\pdfoutput=1

\documentclass[11pt]{article}

\usepackage[final]{acl}

\usepackage{times}
\usepackage{latexsym, amsthm, bbding}
\usepackage{appendix, subfigure, mathtools}
\usepackage{tcolorbox, multicol, multirow, makecell, amsfonts, booktabs, colortbl, tabularx}
\usepackage{algorithm, algpseudocode}

\tcbuselibrary{skins, breakable, theorems}
\usepackage[T1]{fontenc}

\usepackage[utf8]{inputenc}

\usepackage{microtype}
\usepackage{textcomp}
\usepackage{float}
\usepackage{inconsolata}
\usepackage{url}

\usepackage{graphicx}

\title{Free Energy-Driven Reinforcement Learning with Adaptive Advantage Shaping for Unsupervised Reasoning in LLMs}

\author{
  Yiming Huang\textsuperscript{1}, 
  Zhenbo Shi\textsuperscript{1}, 
  Xincheng Wen\textsuperscript{1}, 
  Jichuan Zeng\textsuperscript{3}, 
  Cuiyun Gao\textsuperscript{1}\thanks{Corresponding authors.}, 
  Peiyi Han\textsuperscript{1,2}, 
  Chuanyi Liu\textsuperscript{1,2}\footnotemark[1]
  \\
  \textsuperscript{1}Harbin Institute of Technology, Shenzhen \\
  \textsuperscript{2}Peng Cheng Laboratory \\
  \textsuperscript{3}The Chinese University of Hong Kong \\
  \small \textbf{Correspondence:} 
  \href{mailto:24b951042@stu.hit.edu.cn,2023311604@stu.hit.edu.cn,xiamenwxc@foxmail.com,jczeng@cse.cuhk.edu.hk}{
    \{24b951042,2023311604\}@stu.hit.edu.cn, xiamenwxc@foxmail.com, jczeng@cse.cuhk.edu.hk, 
  }
  \\ 
  \small \href{mailto:gaocuiyun@hit.edu.cn,hanpeiyi@hit.edu.cn,liuchuanyi@hit.edu.cn}{
    \{gaocuiyun,hanpeiyi,liuchuanyi\}@hit.edu.cn
  }
}

\begin{document}
\maketitle
\begin{abstract}

Unsupervised reinforcement learning (RL) has emerged as a promising paradigm for enabling self-improvement in large language models (LLMs). However, existing unsupervised RL-based methods often lack the capacity to adapt to the model's evolving reasoning capabilities during training. Therefore, these methods can misdirect policy optimization in the absence of ground-truth supervision. To address this issue, we introduce \textbf{FREIA}, a novel RL-based algorithm built on two key innovations: (1) \textit{Free Energy-Driven Reward (FER)} adapts rewards to balance consensus and exploration based on the Free Energy Principle. (2) \textit{Adaptive Advantage Shaping (AAS)} adaptively adjusts learning signals based on the statistical characteristics of sampled rewards. Empirical evaluations on nine datasets across three reasoning tasks showcase that FREIA outperforms other unsupervised RL-based baselines. Notably, in mathematical reasoning tasks, FREIA surpasses other methods by an average of 0.5 to 3.5 points in Pass@1 using the DeepSeek-R1-Distill-Qwen-1.5B model.

\end{abstract}

\section{Introduction}

Reinforcement Learning with Verifiable Rewards (RLVR) \cite{wen2025reinforcement} 
has emerged as a fundamental technique for enhancing the reasoning capabilities of Large Language Models (LLMs) \cite{bai2025qwen2, guo2025deepseek}. 
Through alignment with ground-truth supervision signals, RLVR has demonstrated substantial improvements across a wide range of reasoning tasks \cite{chen2025bridging, wang2025geometryzero, yu2018spider}. Nevertheless, this paradigm critically depends on external supervision, which imposes significant data annotation costs requiring human expertise.
As a result, recent research has increasingly explored methods that enable models to learn from intrinsic signals, a process referred to as \textit{unsupervised self-improvement} \cite{zuo2025ttrl}.

\begin{figure}[t]
    \centering
    \includegraphics[width=\linewidth]{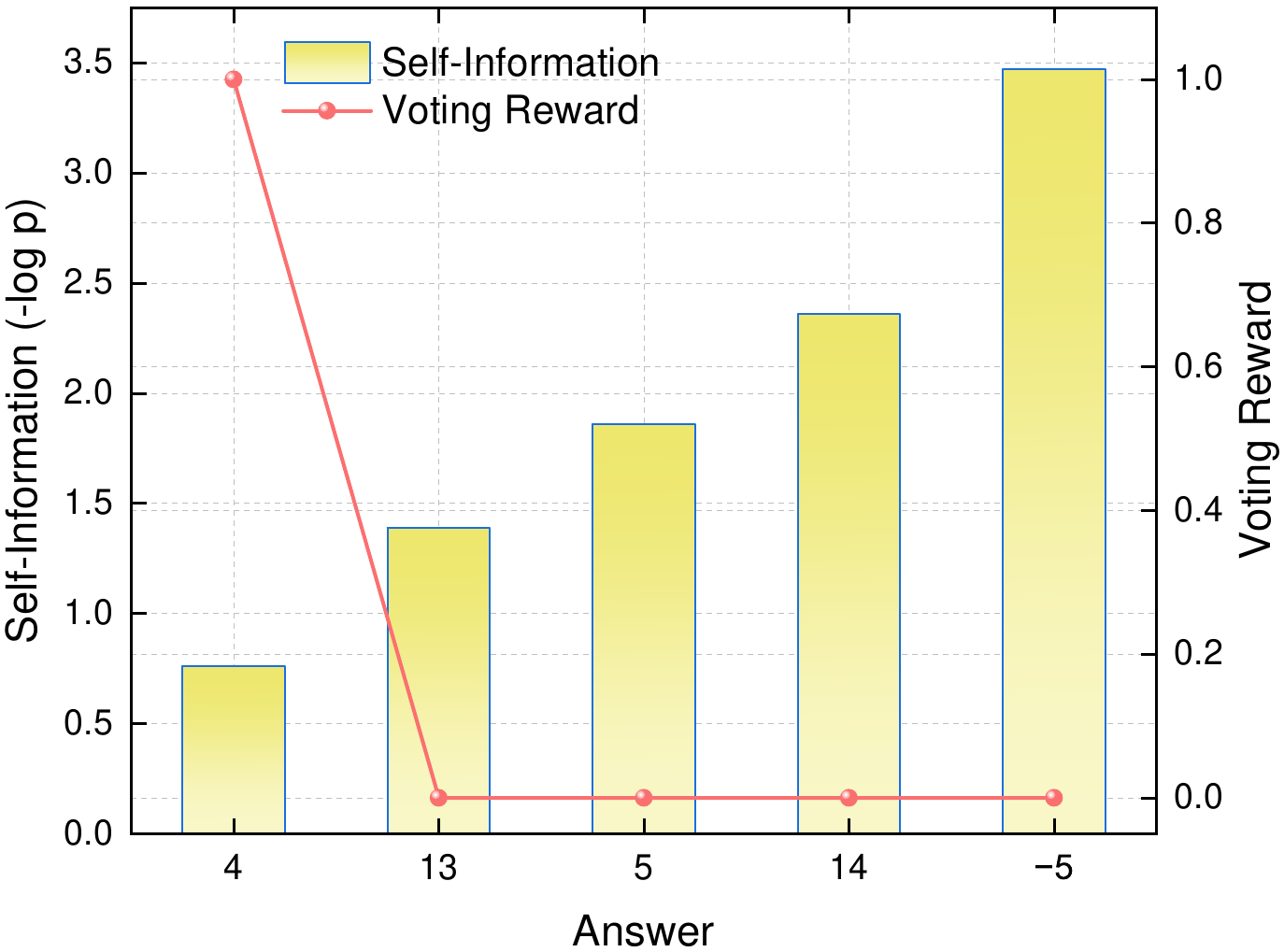}
    \caption{An analysis of reward signals for a math problem, where the correct answer is ``13''. Note that $p$ indicates the frequency of sampled answers.}
    \label{fig:fer_motivation_example}
\end{figure}

\begin{figure*}[t]
\centering
\includegraphics[width=\textwidth]{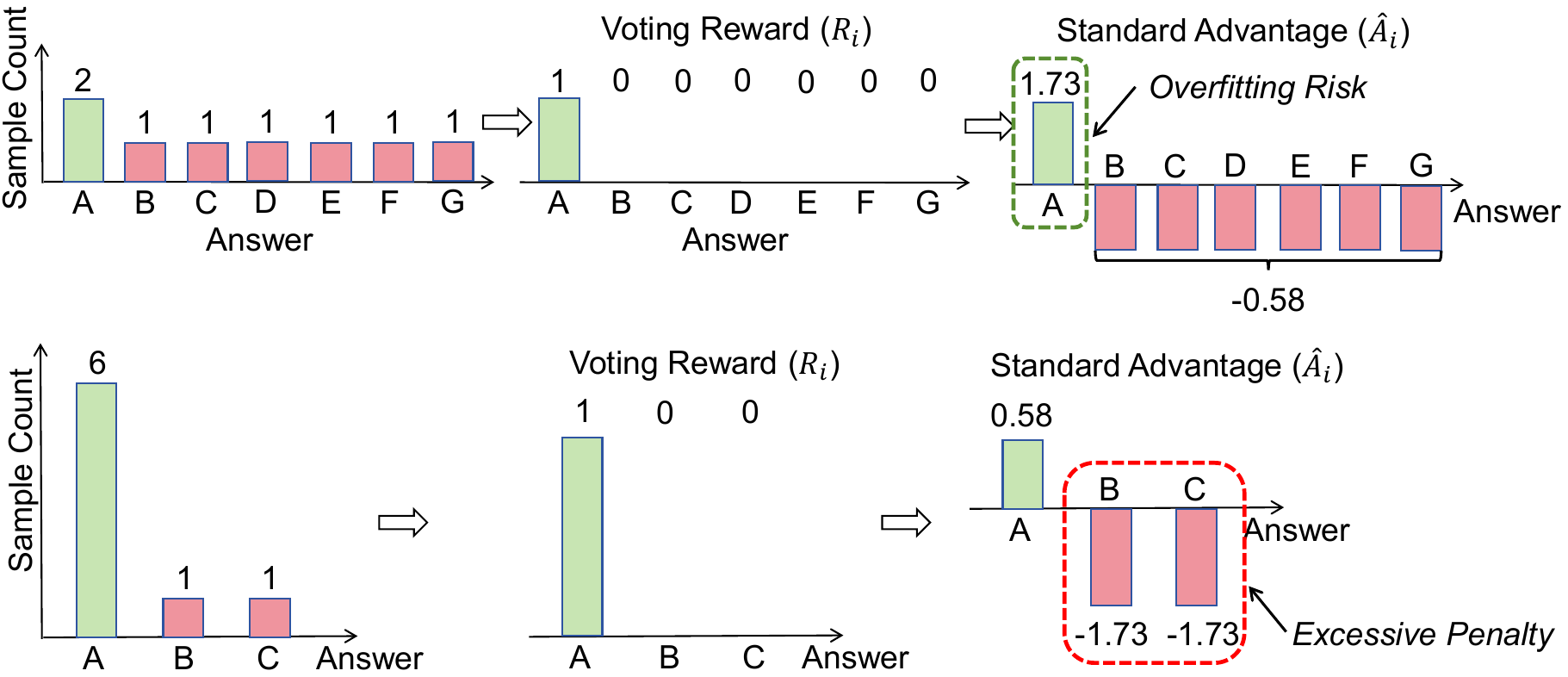}
\caption{An analysis of standard advantage shaping.
\textbf{(Top: Weak Consensus)} The high-reward answer appears infrequently. Standard advantage assigns excessive values to these rare answers.
\textbf{(Bottom: Strong Consensus)} The majority answer dominates the population. Standard advantage excessively penalizes the occasional deviations.}
\label{fig:acra_motivation_binary}
\end{figure*}

A key challenge in promoting unsupervised self-improvement lies in the formulation of reliable learning signals. To cope with this, existing unsupervised methods can be categorized into \textit{trajectory-intrinsic methods} and \textit{population-based methods}. Specifically, \textit{trajectory-intrinsic methods} derive rewards from intrinsic measures of reasoning paths such as model uncertainty \cite{prabhudesai2025maximizing, li2025confidence, garwal2025unreasonable, zhao2025learning, zhang2025no}. 
In contrast, \textit{population-based methods} produce reward signals by comparing each reasoning output against a set of candidate outputs to exploit collective consensus. They assume that majority agreement reflects correctness  \cite{prasad2024self, zuo2025ttrl, liu2025ettrl, yuan2025wisdom, zhang2025co}.  
However, both paradigms cannot adapt to the model's evolving capabilities during training, thereby undermining effective policy optimization.

First, current methods often misdirect policy optimization by applying static criteria. Specifically, \textit{population-based methods} rely solely on consensus. They discard valuable information from minority reasoning paths and impede necessary exploration, particularly during early training stages. As shown in Figure \ref{fig:fer_motivation_example}, majority voting assigns zero reward to the correct answer (``13'') while granting the maximum reward to the incorrect majority (``4''). This bias reinforces erroneous consensus and increases vulnerability to incorrect agreement in unsupervised settings \cite{huang2025low}.  
In contrast, \textit{trajectory-intrinsic methods} risk reinforcing high‑confidence but incorrect answers. This risk is particularly pronounced when the model’s reasoning ability is limited, leading to a significant mismatch between model confidence and actual correctness \cite{zhang2025no}. Figure \ref{fig:fer_motivation_example} illustrates that approaches relying solely on self‑information assign high rewards to rare answers without considering correctness. This bias disrupts robust policy optimization by rewarding incorrect solutions\footnote{The problem is to find the integer $0 \leq n < 18$ such that $n \equiv -11213141$ (mod $18$).}.  

Beyond these design flaws, a second critical issue emerges in how these rewards are processed. Existing unsupervised RL-based methods apply static advantage shaping despite the continuous changes of reward distributions during training. As depicted in Figure~\ref{fig:acra_motivation_binary} (Top), training often begins in a \textit{Weak Consensus} phase, where high‑reward answers are sparse. In this phase, standard advantage normalization assigns large positive advantages to infrequent high‑reward samples. Without label verification, this bias risks overfitting to potentially misleading signals and early convergence before exploring all possible solutions. As the model reaches a \textit{Strong Consensus} (Figure~\ref{fig:acra_motivation_binary}, Bottom), the same strategy assigns excessive penalties to occasional low-reward samples. This shifts the policy optimization focus from refining the dominant consensus to merely avoiding errors. 

As a result, addressing these issues necessitates solving a key research problem: 

\begin{tcolorbox}[colback=black!5!white, colframe=black, 
                  title=\textcolor{black}{\textbf{Key Research Problem}}, 
                  coltitle=white, colbacktitle=pink, breakable,
                  boxrule=1pt, width=0.48\textwidth]
\textbf{\textit{How to design an unsupervised RL-based algorithm that adaptively regulates policy optimization to align with the model's evolving learning dynamics?}}
\end{tcolorbox}

To answer this question, we propose a new algorithm: \textit{Free Energy-Driven Reinforcement Learning with Adaptive Advantage Shaping} (\textbf{FREIA}). This algorithm is designed to enhance the reasoning capabilities of LLMs without external supervision. FREIA integrates two key innovations: (1) \textbf{Free Energy-Driven Reward (FER)} formulates a unified reward based on the \textit{Free Energy Principle (FEP)}. FER views self‑improvement as a process of minimizing free energy. The objective is decomposed into consensus alignment and exploration of novel reasoning paths. By optimizing this objective, the model leverages its internal uncertainty to mitigate early convergence while reinforcing high-quality reasoning paths as confidence grows. (2) \textbf{Adaptive Advantage Shaping (AAS)} adjusts policy updates based on the model's evolving learning dynamics. By analyzing the skewness of reward distribution, AAS identifies the current learning phase and adaptively adjusts advantage signals. This approach reduces the impact of unreliable outliers during early exploration and prevents strict penalties for occasional reasoning paths during convergence. Collectively, these innovations enable FREIA to achieve effective self-improvement in unsupervised settings. We evaluate FREIA on nine benchmarks covering three reasoning tasks. Extensive experiments demonstrate its superiority over other unsupervised baselines.

In summary, this work delivers the following contributions:

(1) We identify a fundamental misalignment in existing unsupervised RL-based methods. First, current reward designs fail to adapt to the model's evolving reasoning capabilities. Second, current advantage estimation ignores the shifting distribution of reward signals during training.

(2) We introduce FREIA, a novel unsupervised RL-based algorithm featuring two key innovations. \textit{Free Energy-Driven Reward (FER)} balances consensus and exploration based on the Free Energy Principle. Moreover, \textit{Adaptive Advantage Shaping (AAS)} adaptively modulates advantage estimation based on real-time reward distributions.

(3) Extensive experimental results confirm that FREIA surpasses other unsupervised baselines on nine benchmarks across three reasoning tasks. This showcases its efficacy in enhancing reasoning performance without external supervision.

\begin{figure*}[t]
\centering
\includegraphics[width=\textwidth]{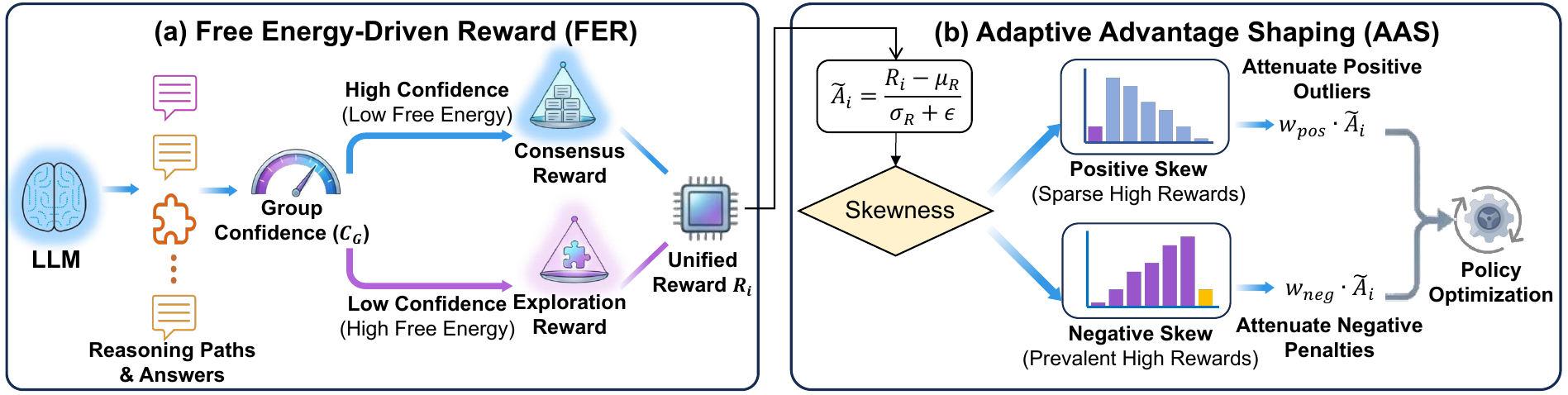}
\caption{The overall framework of FREIA, including Free Energy-Driven Reward (FER) and Adaptive Advantage Shaping (AAS).}
\label{fig:method_overview}
\end{figure*}

\section{Related Work}

Recent efforts to enable self-improvement in LLMs without ground-truth supervision have catalyzed research into unsupervised RL. Existing methods fall into two main paradigms: \textit{trajectory-intrinsic methods} and \textit{population-based methods}.

\noindent \textbf{Trajectory-Intrinsic Methods.} These methods assign rewards to individual reasoning paths using intrinsic metrics (\textit{e.g.}, semantic entropy) to estimate self-confidence \cite{garwal2025unreasonable, li2025confidence, prabhudesai2025maximizing, zhao2025learning, zhang2025no}. While these approaches aim to reduce uncertainty, they often suffer from self-reinforcement bias. Notably, high-confidence but incorrect answers often receive large rewards in the absence of external verification or peer comparison. Additionally, when model confidence is miscalibrated, these methods may exacerbate systematic errors by reinforcing flawed reasoning patterns. They also lack mechanisms to distinguish between genuine certainty derived from accurate reasoning and overconfidence caused by erroneous correlations, leading to instability throughout the training process.

\noindent \textbf{Population-Based Methods.} These methods evaluate answers by comparing each output against other sampled candidates, under the assumption that collective agreement correlates with correctness. Existing strategies can be divided into two levels of granularity: \textit{Final-Answer Consensus} relies on majority voting over final outputs \cite{prasad2024self, zuo2025ttrl, liu2025ettrl, yuan2025wisdom}. This approach often produces stable but rigid learning signals, ignoring valuable minority paths that can lead to correct solutions in later training stages. In contrast, \textit{Cross-View Consensus} \cite{zhang2025co} requires agreement across perturbed input variants to improve robustness. However, it often suppresses exploratory behaviors that are critical in complex reasoning tasks. More importantly, both approaches remain outcome-centric, assigning rewards only based on final agreement.

This analysis underscores the need for a unified and principled framework that integrates the strengths of both paradigms while mitigating their respective limitations. Accordingly, this enables more reliable and adaptive model learning in the absence of ground-truth supervision.

\section{Methodology}\label{sec:method}
As shown in Figure \ref{fig:method_overview} and Appendix \ref{appendix:pesudo}, FREIA integrates Free Energy-Driven Reward (FER) and Adaptive Advantage Shaping (AAS).

\subsection{Free Energy-Driven Reward (FER)}
\label{ssec:bcr}

Existing unsupervised RL-based methods typically produce single-dimensional rewards that misalign with the model's learning dynamics. These methods lack the flexibility to adjust the balance between consensus and exploration as the model improves. To address this, we propose \textbf{Free Energy-Driven Reward (FER)}, which is grounded in the \textit{Free Energy Principle (FEP)} \cite{friston2010free, buckley2017free}. Specifically, FER formulates self‑improvement as minimizing free energy. This objective is decomposed into consensus alignment and exploration of novel reasoning paths.

Specifically, given an input $x$, the model samples a set of $G$ reasoning paths $Y = \{y_1, y_2, \dots, y_G\}$ and extracts final answers $A = \{a_1, ..., a_G\}$. Let $U=\{u_1, u_2, \dots, u_M\}$ denote the set of unique answers within $A$. The frequency distribution of these unique answers $D = \{f_1, f_2, \dots, f_M\}$ serves as empirical input for the subsequent steps.

\textbf{Step 1:} 
Each unique answer $u_i$ is regarded as a hypothesis about the ground truth. To distinguish robust consensus from stochastic noise, we apply a \textit{Non-linear Belief Sharpening} mechanism. Using the frequency $f_i$, the belief weight $w_i$ is given as:
\begin{equation}
    w_i = \text{Softmax}(\alpha \cdot \log(f_i)) = \frac{f_i^\alpha}{\sum_{k=1}^M f_k^\alpha}
    \label{eq:posterior_weight}
\end{equation}
where $W = \{w_1, ..., w_M\}$ constitutes the refined belief distribution over candidate answers.

\textbf{\textit{Analysis.}} 
Since the true ground-truth distribution is inaccessible in unsupervised settings, an \textit{empirical belief} is derived from the sampled response set. From the perspective of FEP, $\alpha$ modulates the model's confidence in the current batch, forcing the distribution to focus on the leading consensus when $\alpha$ increases. Further analysis on $\alpha$ is shown in Appendix \ref{sec:belief}.

\textbf{Step 2:} 
We define \textit{Group Confidence} $C_G$ to quantify the degree of group consensus for each training sample:
\begin{equation}
    C_G = 
    \begin{cases} 
    1.0 & \text{if } M = 1 \\
    1 - \frac{H(W)}{\log M} = 1 - \frac{-\sum_{j} w_j \log w_j}{\log M} & \text{if } M > 1 
    \end{cases}
    \label{eq:group_confidence}
\end{equation}
where $H(\cdot)$ denotes the Shannon entropy. $C_G \to 1$ signifies high certainty, whereas $C_G \to 0$ indicates high uncertainty. The superiority of $C_G$ is further examined in Appendix \ref{app:mixing_ablation}.

\textbf{Step 3:} 
The total reward for each $y_i$ is formulated as an adaptive trade-off between \textit{Consensus} and \textit{Exploration}, modulated by the group confidence $C_G$. Specifically, \textit{Consensus} encourages the model to align with the most probable solution:
\begin{equation}
    r_{\text{cons}}(y_i) = 
    \begin{cases}
        1.0 & \text{if } a_i = \text{Vote}(\{a_k\}_{k=1}^G) \\
        0.0 & \text{otherwise}
    \end{cases}
    \label{eq:exploit_reward}
\end{equation}

Moreover, \textit{Exploration} incentivizes the discovery of diverse solutions when uncertainty is high. For a given $y_i$ producing answer $a_i$ with belief weight $w_i$, the reward is defined as:
\begin{equation}
    r_{\text{explore}}(y_i) = \tanh (-\log w_i)
    \label{eq:explore_reward}
\end{equation}
where the $\tanh$ function serves as a soft normalization mechanism. This strictly bounds $r_{\text{explore}}$ within $(0,1)$, preventing the exploration signal from dominating the policy optimization process.

\begin{figure}[h!]
\centering
\includegraphics[width=\linewidth]{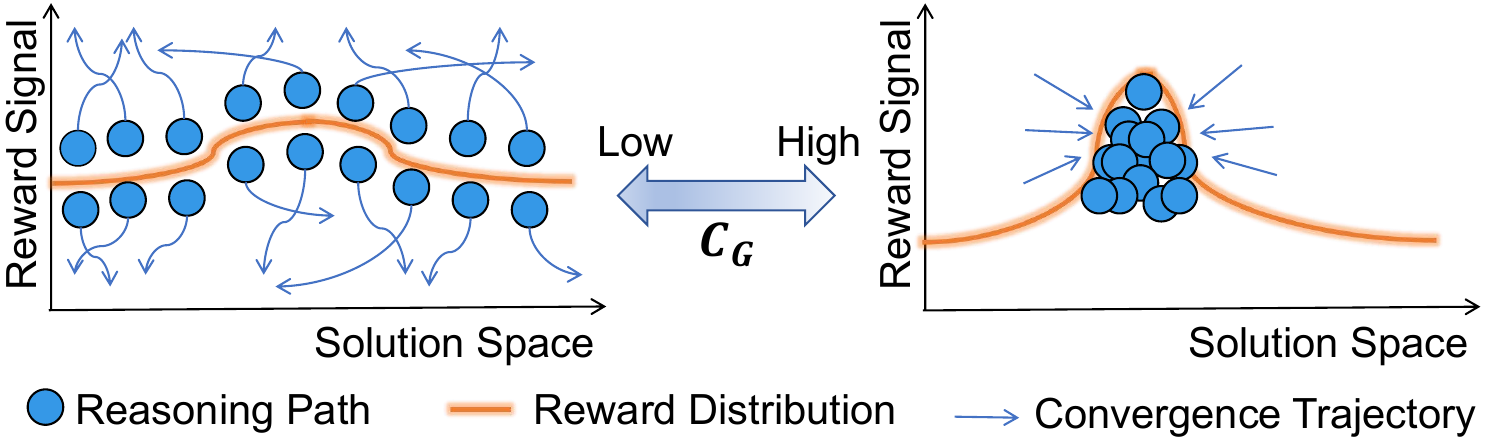}
\caption{Visualization of FER. Specifically, the left part depicts the Exploration term when $C_G$ is low, while the right part shows the Consensus term under high $C_G$.
}
\label{fig:fer}
\end{figure}

As illustrated in Figure \ref{fig:fer}, the final reward integrates the \textit{Consensus} and \textit{Exploration} components modulated by $C_G$, which is given as:
\begin{equation}
    R_i = C_G \cdot r_{\text{cons}}(y_i) + (1 - C_G) \cdot r_{\text{explore}}(y_i)
    \label{eq:bcr_final}
\end{equation}

Eq. (\ref{eq:bcr_final}) adaptively regulates the learning process, which prioritizes consolidating the consensus in high-confidence states ($C_G \to 1$), while encouraging exploration in low-confidence states ($C_G \to 0$). Further analysis of FER is shown in Appendix \ref{sec:fep_derivation}.

\subsection{Adaptive Advantage Shaping (AAS)}
\label{ssec:aas}

Standard advantage functions are suboptimal for self-improvement in unsupervised settings. Therefore, we propose \textbf{Adaptive Advantage Shaping (AAS)}, which adaptively modulates advantages based on the reward distributional characteristics.

\textbf{Step 1:} For each $y_i$, the standard advantage $\tilde{A}_i$ is derived from the FER scores $\{R(y_k)\}_{k=1}^G$:
\begin{equation}
    \tilde{A}_{i} = \frac{R_i - \mu_R}{\sigma_R+\epsilon}
    \label{eq:standard_advantage}
\end{equation}
where $\mu_R$ and $\sigma_R$ are the mean and standard deviation of the FER scores, respectively. $\epsilon$ is a small constant for numerical stability.

\textbf{Step 2: } The learning dynamics of the model are quantified by the \textit{skewness} of the reward distribution $\{R(y_k)\}_{k=1}^G$. A positive skew indicates a distribution dominated by low rewards, whereas a negative skew signifies a prevalence of high rewards. The sample skewness $\mathcal{S}$ is defined as (See Appendix \ref{sec:skewness} for additional analysis):
\begin{equation}
    \mathcal{S} = \frac{1}{G} \sum_{i=1}^G \left(\frac{R_i - \mu_R}{\sigma_R + \epsilon} \right)^3
    \label{eq:skewness}
\end{equation}

\textbf{Step 3:} 
Adaptive weights for positive and negative advantages are computed from $\mathcal{S}$ to mitigate potential biases in standard normalization:
\begin{equation}
    w_{\text{pos}} = \sigma(-\mathcal{S}), ~w_{\text{neg}} = \sigma(\mathcal{S})
\end{equation}
where $\sigma(\cdot)$ denotes the sigmoid function.

\textbf{\textit{Analysis.}} AAS functions as a stabilizer against unsupervised distributional noise:

\textit{Case 1: Positive Skew.} In this case, standard normalization assigns excessively high advantages to rare winning paths. However, rewards are computed without ground‑truth labels in unsupervised settings, indicating that high rewards do not necessarily correspond to correct or generalizable reasoning. Since such paths may represent stochastic outliers rather than robust solutions, AAS \textit{attenuates positive advantages} ($w_{\text{pos}} \to 0$). This ensures a cautious policy update, preventing the model from overfitting to potentially misleading signals. 

\textit{Case 2: Negative Skew.} In this case, standard normalization assigns large penalties to occasional deviations. Under unsupervised conditions, the absence of ground-truth supervision also makes it unclear whether those deviations are genuine reasoning errors or harmless variations. As these deviations may not reflect fundamental mistakes, AAS \textit{attenuates negative advantages} ($w_{\text{neg}} \to 0$). This ensures that the model does not over-correct based on rare low-reward samples. Additional analysis on AAS is shown in Appendix \ref{appendix:theoretical_guarantees}.

\textbf{Step 4: } The final advantage $\hat{A}_i$ used in the policy update is modulated by these adaptive weights:
\begin{equation}
    \hat{A}_i =
    \begin{cases}
        w_{\text{pos}} \cdot \tilde{A}_i & \text{if } \tilde{A}_i > 0 \\
        w_{\text{neg}} \cdot \tilde{A}_i & \text{if } \tilde{A}_i < 0
    \end{cases}
    \label{eq:acra_final}
\end{equation}

This shaped advantage $\hat{A}_i$ is integrated into the standard GRPO objective \cite{shao2024deepseekmath}:
\begin{equation}
    \begin{aligned}
    \mathcal{L}(\theta) =& 
    \mathbb{E} \Big[\frac{1}{G} \sum_{i=1}^G \frac{1}{|o_i|} \sum_{t=1}^{|o_i|} 
    \min\Big( r_{i,t}(\theta) \hat{A}_{i},\,  \\
    &\text{clip}(r_{i,t}(\theta), 1-\epsilon, 1+\epsilon) \hat{A}_{i} \Big) \\
    &- \beta D_{\text{KL}}(\pi_\theta \| \pi_{\text{ref}})\Big]
    \end{aligned}
    \label{eq:GRPO}
\end{equation}

\begin{table*}[t]
\centering
\small
\setlength{\tabcolsep}{4.9mm}
\begin{tabular}{lcccccc}
\toprule
\textbf{Dataset} & \textbf{Base} & \textbf{GRPO (Supervised)} & \textbf{TTRL} & \textbf{Entropy} & \textbf{Intuitor} & \textbf{FREIA} \\
\midrule

\rowcolor{green!10} \multicolumn{7}{c}{\textit{Qwen2.5-Math-1.5B-Instruct}} \\
\midrule
MATH500  & 74.2 & 75.2$_{\pm 0.7}$ & 75.0$_{\pm 0.4}$ & 74.6$_{\pm 0.8}$ & 74.4$_{\pm 0.8}$ & \textbf{75.4}$_{\pm 0.2}$ \\
AIME24   & 10.0 & \textbf{13.3}$_{\pm 1.9}$ & \textbf{13.3}$_{\pm 0.0}$ & 10.0$_{\pm 0.0}$ & 10.0$_{\pm 0.0}$ & \textbf{13.3}$_{\pm 0.0}$ \\
AIME25   & 3.3  & \textbf{16.7}$_{\pm 0.0}$ & \textbf{16.7}$_{\pm 0.0}$ & 13.3$_{\pm 0.0}$ & 10.0$_{\pm 0.0}$ & \textbf{16.7}$_{\pm 0.0}$ \\
AMC23    & 47.5 & \textbf{52.5}$_{\pm 1.4}$ & \textbf{52.5}$_{\pm 0.0}$ & 47.5$_{\pm 0.0}$ & 47.5$_{\pm 0.0}$ & \textbf{52.5}$_{\pm 0.0}$ \\
Minerva  & 28.7 & 31.3$_{\pm 0.9}$ & 30.9$_{\pm 0.5}$ & 29.4$_{\pm 0.8}$ & 28.7$_{\pm 0.9}$ & \textbf{32.0}$_{\pm 0.4}$ \\
Olympiad & 35.2 & 40.8$_{\pm 0.6}$ & 40.1$_{\pm 0.4}$ & 39.2$_{\pm 0.7}$ & 38.4$_{\pm 0.7}$ & \textbf{41.2}$_{\pm 0.3}$ \\
\textit{Avg.}     & 33.2 & 38.3$_{\pm 0.9}$ & 38.1$_{\pm 0.2}$ & 35.7$_{\pm 0.4}$ & 34.8$_{\pm 0.4}$ & \textbf{38.5}$_{\pm 0.2}$ \\
\midrule

\rowcolor{blue!10} \multicolumn{7}{c}{\textit{Qwen2.5-3B-Instruct}} \\
\midrule
MATH500  & 62.0 & 66.0$_{\pm 0.6}$ & \textbf{66.6}$_{\pm 0.5}$ & 64.6$_{\pm 0.7}$ & 64.0$_{\pm 0.7}$ & 65.2$_{\pm 0.3}$ \\
AIME24   & 0.0  & \textbf{10.0}$_{\pm 0.0}$ & 6.7$_{\pm 0.0}$ & 3.3$_{\pm 1.9}$ & 3.3$_{\pm 0.0}$ & \textbf{10.0}$_{\pm 0.0}$ \\
AIME25   & 0.0  & \textbf{10.0}$_{\pm 0.0}$ & 6.7$_{\pm 0.0}$ & 3.3$_{\pm 0.0}$ & 3.3$_{\pm 0.0}$ & \textbf{10.0}$_{\pm 0.0}$ \\
AMC23    & 35.0 & 37.5$_{\pm 1.4}$ & \textbf{40.0}$_{\pm 0.0}$ & 37.5$_{\pm 0.0}$ & 37.5$_{\pm 0.0}$ & 37.5$_{\pm 0.0}$ \\
Minerva  & 24.3 & 25.4$_{\pm 0.8}$ & 25.8$_{\pm 0.5}$ & 25.0$_{\pm 0.9}$ & 24.3$_{\pm 0.9}$ & \textbf{26.1}$_{\pm 0.4}$ \\
Olympiad & 29.1 & 31.5$_{\pm 0.6}$ & 31.2$_{\pm 0.4}$ & 30.0$_{\pm 0.7}$ & 29.8$_{\pm 0.8}$ & \textbf{31.9}$_{\pm 0.3}$ \\
\textit{Avg.}     & 25.1 & \textbf{30.1}$_{\pm 0.6}$ & 29.5$_{\pm 0.2}$ & 27.3$_{\pm 0.7}$ & 27.0$_{\pm 0.4}$ & \textbf{30.1}$_{\pm 0.2}$ \\
\midrule

\rowcolor{orange!10} \multicolumn{7}{c}{\textit{DeepSeek-R1-Distill-Qwen-1.5B}} \\
\midrule
MATH500  & 77.6 & 82.4$_{\pm 0.4}$ & \textbf{82.6}$_{\pm 0.3}$ & 81.8$_{\pm 0.6}$ & 81.4$_{\pm 0.6}$ & 82.2$_{\pm 0.2}$ \\
AIME24   & 16.7 & \textbf{20.0}$_{\pm 0.0}$ & \textbf{20.0}$_{\pm 0.0}$ & 16.7$_{\pm 0.0}$ & 16.7$_{\pm 0.0}$ & \textbf{20.0}$_{\pm 0.0}$ \\
AIME25   & 16.7 & \textbf{20.0}$_{\pm 1.9}$ & \textbf{20.0}$_{\pm 0.0}$ & 16.7$_{\pm 0.0}$ & 16.7$_{\pm 0.0}$ & \textbf{20.0}$_{\pm 0.0}$ \\
AMC23    & 62.5 & 70.0$_{\pm 0.0}$ & 70.0$_{\pm 0.0}$ & 65.0$_{\pm 0.0}$ & 65.0$_{\pm 0.0}$ & \textbf{72.5}$_{\pm 0.0}$ \\
Minerva  & 27.6 & 30.5$_{\pm 0.6}$ & 30.9$_{\pm 0.4}$ & 29.8$_{\pm 0.7}$ & 29.4$_{\pm 0.7}$ & \textbf{31.3}$_{\pm 0.3}$ \\
Olympiad & 42.4 & 48.6$_{\pm 0.5}$ & 49.0$_{\pm 0.3}$ & 47.5$_{\pm 0.6}$ & 46.6$_{\pm 0.6}$ & \textbf{49.4}$_{\pm 0.3}$ \\
\textit{Avg.}     & 40.6 & 45.3$_{\pm 0.6}$ & 45.4$_{\pm 0.2}$ & 42.7$_{\pm 0.3}$ & 42.4$_{\pm 0.3}$ & \textbf{45.9}$_{\pm 0.1}$ \\

\bottomrule
\end{tabular}
\caption{Experimental results on multiple mathematical reasoning benchmarks. The results are reported as mean and standard deviation across 3 random seeds (\textit{i.e.}, Mean$_{\pm \text{Std}}$). The best results are highlighted in bold.}
\label{tab:main_results_math}
\end{table*}
\section{Experiments}\label{sec:exp}

Our evaluation focuses on four research questions: \textbf{RQ1:} Can FREIA achieve superior reasoning performance compared to other unsupervised RL-based baselines? \textbf{RQ2:} Does FREIA effectively balance the trade-off between exploration and consensus? \textbf{RQ3:} What are the specific contributions of FER and AAS to the overall effectiveness of FREIA? \textbf{RQ4:} How sensitive is FREIA's performance to its key hyperparameter $\alpha$?

\subsection{Settings}

\noindent \textbf{Models.} To assess mathematical reasoning capabilities, we selected Qwen2.5-Math-1.5B-Instruct \cite{yang2024qwen21}, Qwen2.5-3B-Instruct \cite{yang2024qwen2}, and DeepSeek-R1-Distill-Qwen-1.5B \cite{guo2025deepseek}. Moreover, Qwen2.5-Coder-3B-Instruct \cite{hui2024qwen2} and Qwen2.5-VL-3B-Instruct \cite{bai2025qwen2} were used in SQL generation and multi-modal reasoning, respectively.

\noindent \textbf{Datasets.} Training for mathematical reasoning was conducted using MATH \cite{hendrycks2021measuring}, followed by evaluations on MATH500 \cite{hendrycks2021measuring}, AIME24 \cite{li2024numinamath}, AIME25 \cite{codeforcesamerican}, AMC23 \cite{ouyang2022training}, Minerva \cite{lewkowycz2022solving}, and OlympiadBench \cite{huang2024olympicarena}. For SQL generation, we utilized BIRD-Train \cite{li2024can} for training, with evaluation performed on Spider-Dev \cite{yu2018spider} and BIRD-Dev \cite{li2024can}. We used the training and testing splits of Geometry3K \cite{lu2021inter} for multi-modal reasoning. 

\begin{figure}[h!]
\centering
\includegraphics[width=2.8in]{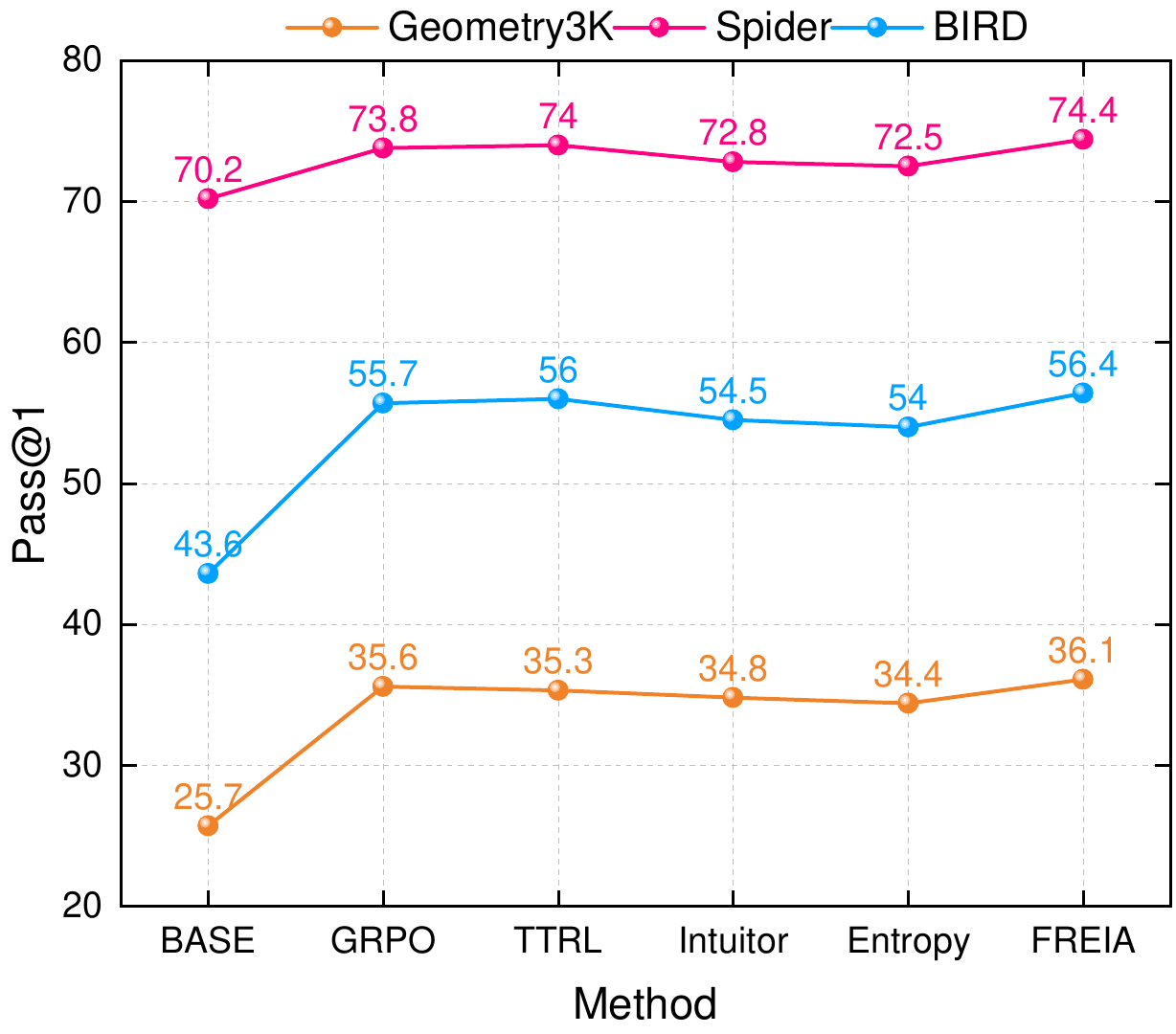}
\caption{Experimental results on SQL generation and multi-modal reasoning.}
\label{fig:geoSQL}
\end{figure}

\begin{figure}[h!]
    \centering
    \includegraphics[width=2.8in]{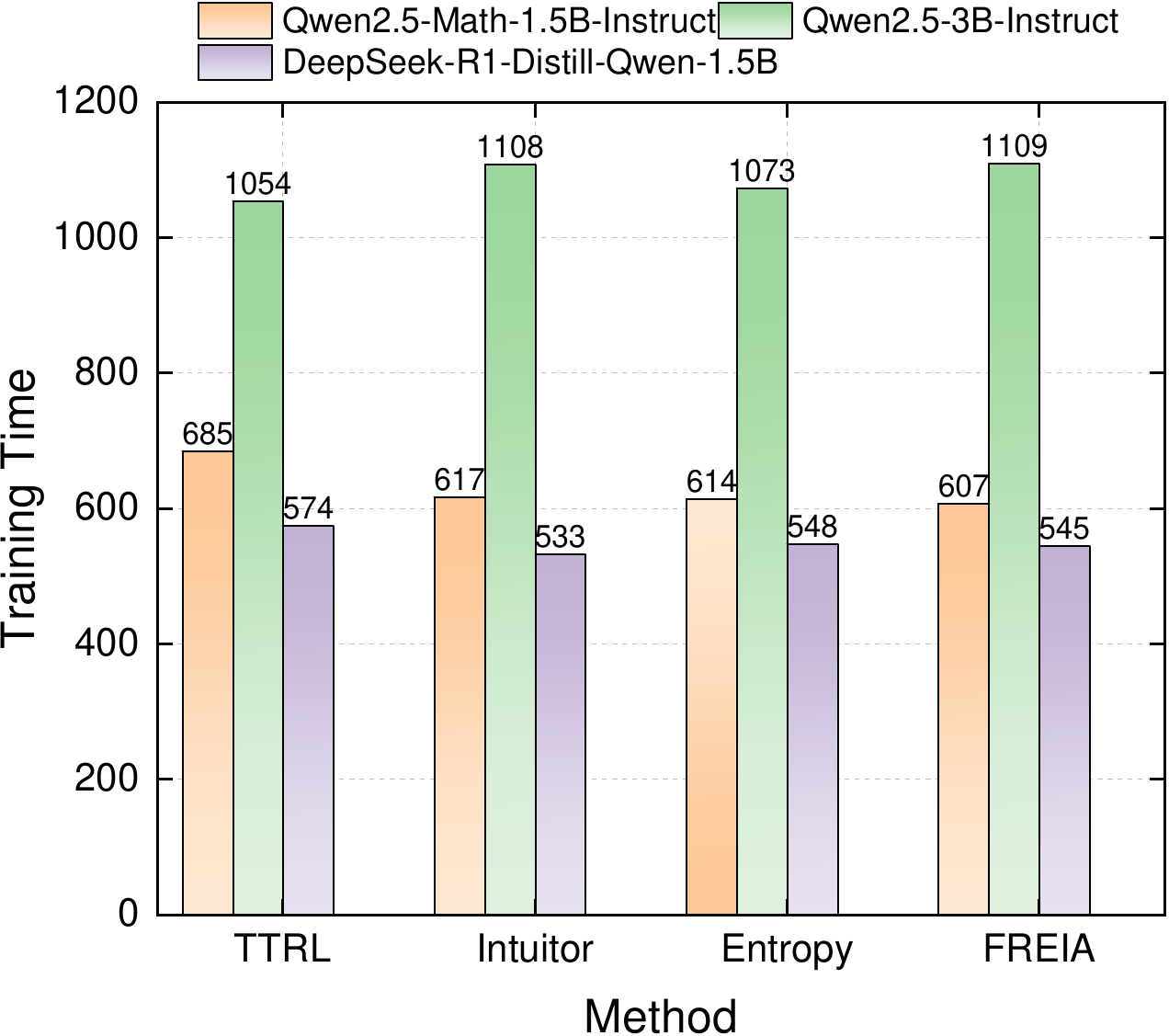}
    \caption{Comparison of total training wall-clock time (minutes) using various methods.}
    \label{fig:time}
\end{figure}

\noindent \textbf{Implementation Details.}
In this work, the coefficient of KL loss term was $\beta=0.001$. The batch size and the number of rollouts were $512$ and $8$, and we used a sampling temperature of $1.0$. We adopted AdamW \cite{zhou2024towards} with a learning rate of $1 \times 10^{-6}$ and trained for $400$ steps. For FREIA, the parameter $\alpha$ was $2$. During evaluation, the sampling temperature was $0.6$, and Pass@1 was used for evaluation. We performed experiments using four NVIDIA GeForce A100 40GB GPUs. FREIA was compared against base model, GRPO (supervised RL) \cite{shao2024deepseekmath}, TTRL \cite{zuo2025ttrl}, Entropy \cite{prabhudesai2025maximizing}, and Intuitor \cite{zhao2025learning}, which were set to the same hyperparameters as FREIA. Further experimental setups are shown in Appendix \ref{appendix:reference}.

\begin{figure*}[t]
\centering
\includegraphics[width=0.95\textwidth]{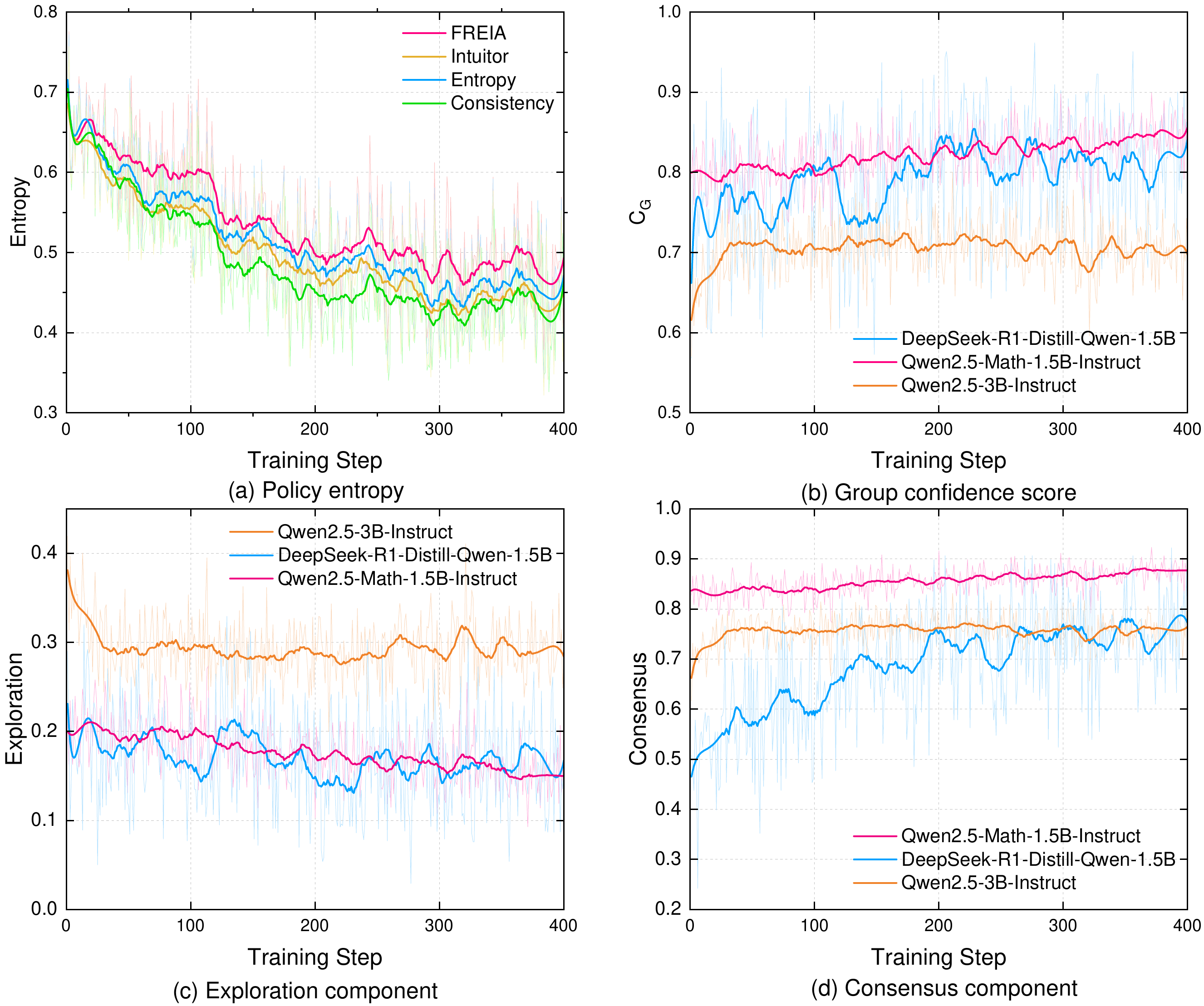}
\caption{
Training dynamics of FREIA. 
(a) Policy entropy using DeepSeek-R1-Distill-Qwen-1.5B; 
(b) Group confidence score ($C_G$); 
(c) Exploration component of the reward; 
(d) Consensus component of the reward.
}
\label{fig:analysis_curves}
\end{figure*}

\subsection{Experimental Results}

\begin{figure*}[t]
\centering
\includegraphics[width=0.95\textwidth]{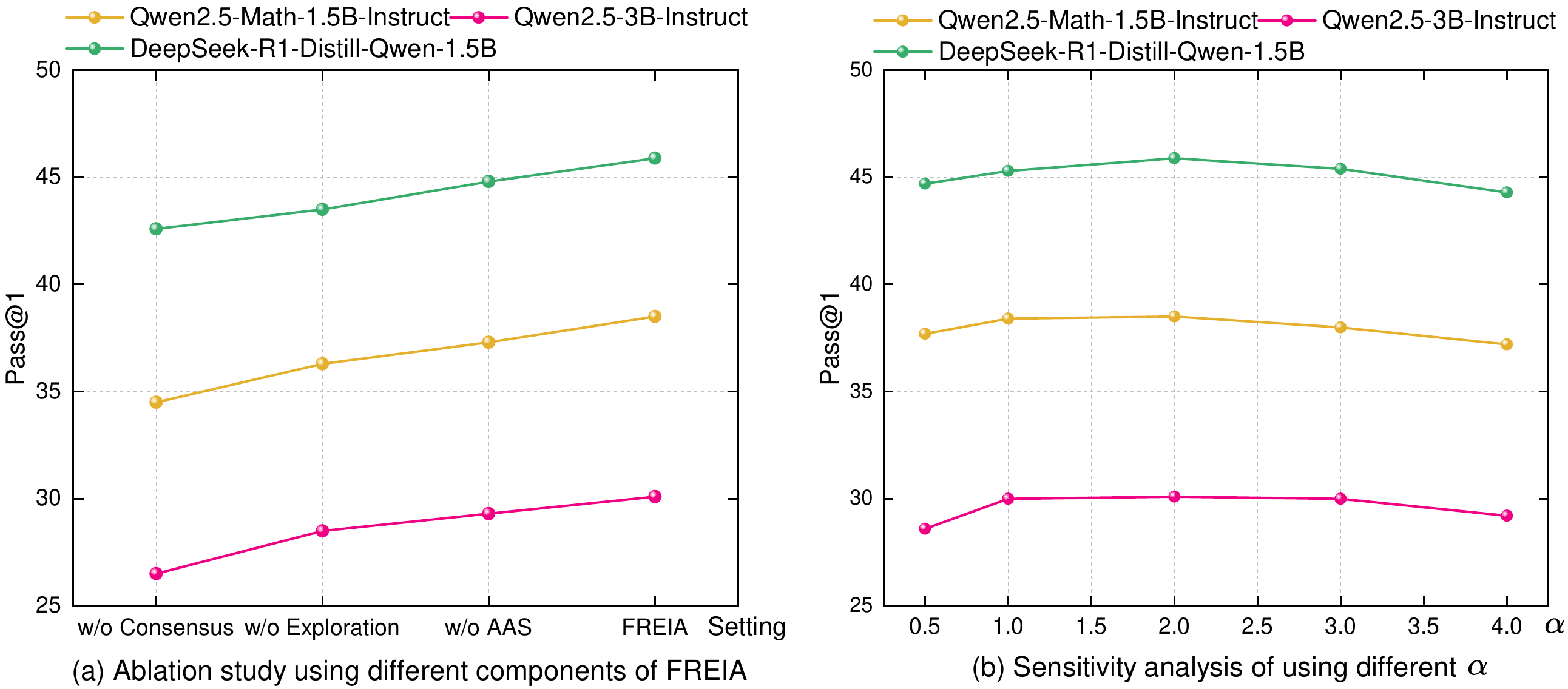}
\caption{
Ablation study and hyperparameter sensitivity analysis of FREIA.
(a) Average Pass@1 of the full FREIA compared to its ablated variants. 
(b) Average Pass@1 under varying values of the hyperparameter $\alpha$.
}
\label{fig:ablation_sensitivity}
\end{figure*}

\subsubsection{Main Results (RQ1)}
\label{ssec:main_results}

As shown in Table \ref{tab:main_results_math}, FREIA achieves superior average performance on mathematical reasoning. Remarkably, FREIA rivals or even surpasses the supervised GRPO. This performance gain is largely attributed to the granularity and density of the learning signal. While GRPO relies on binary feedback offering limited guidance, FER provides \textit{continuous feedback from the trade-off between consensus and exploration}. Therefore, by encouraging diverse reasoning paths aligned with the model's internal knowledge, FREIA acquires robust patterns that generalize better to unseen complexities. This is further evidenced in Figure \ref{fig:geoSQL} and Table \ref{tab:generalization_results}, where FREIA showcases strong transferability to SQL generation and multi-modal reasoning. As a result, minimizing free energy cultivates a more adaptive reasoning capability than outcome supervision.

\subsubsection{Training Dynamics (RQ2)}
\label{ssec:training_dynamics}

To elucidate the mechanisms of FREIA, we analyze the evolution of key metrics during training.

\noindent \textbf{Computational Efficiency.}
Figure \ref{fig:time} compares the total wall-clock training time using different methods. Despite incorporating FER and AAS, FREIA maintains training efficiency analogous to other baselines. This suggests that the overhead from FER and AAS is minimal, establishing FREIA as an efficient framework for unsupervised self-improvement. Further analysis of FREIA's computational efficiency is shown in Appendix \ref{sec:efficiency_analysis}.

\noindent \textbf{Free-Energy-Inspired Training Dynamics.} Figure~\ref{fig:analysis_curves} offers insights into the learning process in FREIA. First, the consistent decrease in policy entropy (Figure~\ref{fig:analysis_curves}(a)) and the upward trend in $C_G$ (Figure~\ref{fig:analysis_curves}(b)) indicate that the model gradually reduces internal uncertainty and concentrates on high-consensus answers.
This is driven by the dynamic interplay between the two FER components.  Initially, the Exploration term promotes the discovery of diverse reasoning paths. As the model's reasoning improves, the Consensus term exhibits an upward trend shown in Figure~\ref{fig:analysis_curves}(d), reinforcing the discovered high-quality solutions. However, the Exploration term in Figure~\ref{fig:analysis_curves}(c) maintains fluctuating throughout training. This confirms that FREIA sustains necessary exploration to mitigate early convergence. Additional analysis on output diversity is provided in Appendix \ref{sec:diversity_analysis}.



\subsubsection{Ablation Study and Sensitivity Analysis (RQ3 \& RQ4)}
\label{ssec:ablation_sensitivity}


\noindent \textbf{Impact of Key Components (RQ3).}
As shown in Figure \ref{fig:ablation_sensitivity}(a), a consistent improvement is observed from ablated variants to FREIA, confirming that each innovation is essential. The most pronounced performance degradation occurs in the \textit{w/o Consensus} variant, suggesting that consensus is a primary driver of self-improvement. The results in \textit{w/o Exploration} variant also decline, highlighting the need of mitigating early convergence. While the \textit{w/o AAS} variant outperforms other ablations, it remains inferior to the full FREIA. This gap demonstrates that static advantage shaping is inadequate for refining the policy optimization process.

\noindent \textbf{Hyperparameter Sensitivity (RQ4).}
Figure \ref{fig:ablation_sensitivity}(b) examines the impact of the parameter $\alpha$, which controls the sensitivity of belief refinement within FER. Specifically, at lower $\alpha$ values, the system exhibits excessive stochasticity, resulting in unstable reasoning and difficulty in consolidating reliable paths. Conversely, excessively high $\alpha$ values strengthens consensus, leading to early convergence toward suboptimal solutions. Crucially, performance reaches its maximum at an intermediate $\alpha$ value, indicating that FREIA effectively balances the trade-off between exploration and consensus. Furthermore, it exhibites strong robustness, obviating the need for fine-grained hyperparameter tuning to achieve competitive performance. Further analysis on the rollout size $G$ and additional results are shown in Appendices \ref{sec:sensitivity_group_size} and \ref{sec:additional results}, respectively.

\subsection{Case Study}
To illustrate the advantages of FREIA over other unsupervised methods, we present a case study in Appendix \ref{sec:case}. Specifically, the problem involves a geometry problem with a strict inequality constraint. Notably, FREIA correctly outputs the solution by systematically enumerating all potential vertex configurations and filtering candidates against the constraint to isolate the valid answer. In contrast, TTRL falls into a reasoning shortcut by assuming a default vertex order, leading to a result that violates the explicit condition ($\textit{i.e.}, x>7$). Entropy and Intuitor fail to reason effectively, degenerating into repetitive loops without performing actual calculations. This demonstrates FREIA’s superior capability in complex reasoning compared to other unsupervised baselines.

\section{Conclusion}\label{sec:con}
In this work, we propose FREIA, a unified framework that integrates unsupervised exploration and reasoning alignment under the Free Energy Principle. It introduces Free Energy‑Driven Reward (FER) and Adaptive Advantage Shaping (AAS) without ground-truth supervision. Extensive experiments across multiple reasoning benchmarks showcase that FREIA promotes self‑improvement and achieves competitive reasoning performance. This approach provides a new paradigm for alignment in scenarios where labeled data is unavailable.

\section*{Limitations}

Although FREIA exhibits strong capabilities, we identify specific avenues for future enhancement. (1) Our current resource budget limited the empirical scope to models with a maximum of 3B parameters. Nevertheless, we anticipate that the principles underlying FREIA’s mechanisms will scale effectively to larger foundation models. (2) The Group Confidence relies exclusively on final answer distributions, neglecting the semantic nuances of intermediate reasoning paths. Developing a more advanced mechanism incorporating process-level monitoring would enhance the model's efficacy. (3) AAS employs batch-level skewness as a proxy for the collective learning state, which may obscure intra-batch variations. A more granular approach adapting to heterogeneous learning dynamics within a single batch represents a valuable direction for future optimization.

\section*{Ethical Considerations}
We are committed to responsible AI research and adhere to the following principles: (1) To guarantee data privacy and reproducibility, we exclusively utilized publicly available datasets for training and evaluation, thereby avoiding any risk to private user data. (2) The study relied entirely on open-source LLMs. While we recognize the environmental impact associated with the energy consumption required for training, we believe improving sample efficiency is a step toward greener AI.

\section*{Acknowledgments} 
This work is supported by the National Key Research and Development Program of China under Grant 2023YFB3106504, the National Natural Science Foundation of China under project (No. 62472126), Shenzhen Science and Technology Program under Grant ZDSYS20210623091809029, the Major Key Project of PCL under Grant PCL2024A04 and PCL2025A16, and CCF-Huawei Populus Grove Fund.

\bibliography{custom}

\newpage
\appendix

\newpage

\section{Appendix A: Additional Experimental Setup} \label{appendix:reference}

\subsection{Models and Benchmarks}
Table \ref{tab:models_and_benchmarks} details the access links for the models and benchmarks utilized in this study. All datasets are publicly available under the CC BY-SA 4.0 license, which enables modifications and inclusion of additional annotations on the original datasets.

\begin{table*}[t]
    \centering
    \small 
    \begin{tabularx}{\textwidth}{l X X}
        \toprule
        \textbf{Name} & \textbf{Access Link} & \textbf{\# Samples} \\
        \midrule
        \multicolumn{3}{c}{\textit{\textbf{Models}}} \\
        \midrule
        Qwen2.5-Math-1.5B-Instruct & \url{https://huggingface.co/Qwen/Qwen2.5-Math-1.5B-Instruct} & - \\
        Qwen2.5-3B-Instruct & \url{https://huggingface.co/Qwen/Qwen2.5-3B-Instruct} & -  \\
        DeepSeek-R1-Distill-Qwen-1.5B & \url{https://huggingface.co/deepseek-ai/DeepSeek-R1-Distill-Qwen-1.5B} & - \\
        Qwen2.5-Coder-3B-Instruct & \url{https://huggingface.co/Qwen/Qwen2.5-Coder-3B-Instruct} & - \\
        Qwen2.5-VL-3B-Instruct & \url{https://huggingface.co/Qwen/Qwen2.5-VL-3B-Instruct} & - \\
        \midrule
        \multicolumn{3}{c}{\textit{\textbf{Benchmarks}}} \\
        \midrule
        MATH & \url{https://huggingface.co/datasets/HuggingFaceH4/MATH} & 12000 training data \\
        MATH500 & \url{https://huggingface.co/datasets/HuggingFaceH4/MATH-500} & 500 evaluation data\\
        AIME24 & \url{https://huggingface.co/datasets/HuggingFaceH4/aime_2024} & 30 evaluation data\\
        AIME25 & \url{https://huggingface.co/datasets/HuggingFaceH4/aime_2025} & 30 evaluation data\\
        AMC23 & \url{https://huggingface.co/datasets/math-ai/amc23} & 40 evaluation data\\
        Minerva & \url{https://huggingface.co/datasets/svc-huggingface/minerva-math} & 272 evaluation data\\
        OlympiadBench & \url{https://huggingface.co/datasets/knoveleng/OlympiadBench} & 674 evaluation data\\
        Spider & \url{https://yale-lily.github.io/spider} & 1034 evaluation data\\
        BIRD & \url{https://bird-bench.github.io/} & 9428 and 1534 data for training and evaluation\\
        Geometry3K & \url{https://huggingface.co/datasets/hiyouga/geometry3k} & 2100 and 601 data for training and evaluation\\
        \bottomrule
    \end{tabularx}
    \caption{Details of models and benchmarks utilized in this study.}
    \label{tab:models_and_benchmarks}
\end{table*}

\subsection{Answer Extraction and Equivalence}
To ensure accurate reward computation, we applied task-specific strategies to extract answers and calculate unique answers.

\paragraph{Mathematical Reasoning and Geometry3K.}
We employed standard answer extraction techniques used in prior work \cite{shao2024deepseekmath}. The final answer was extracted from the answer box and compared via exact string matching.

\paragraph{SQL Generation.}
Since different SQL queries can be semantically equivalent despite having different surface forms, we determined equivalence based on \textit{SQL execution results} rather than string matching. To be specific, two SQL queries were considered the same answer if they yielded identical execution results on the given databases. The set of unique answers was defined by the set of distinct SQL execution outputs.

\section{Appendix B: Pseudo Code for FREIA} \label{appendix:pesudo}

We provide the pseudo code for FREIA in Algorithm \ref{alg:freia}.

\begin{algorithm*}[h!]
\small
\caption{Free Energy-Driven Reinforcement Learning with Adaptive Shaping (FREIA)}
\label{alg:freia}
\begin{flushleft}
\textbf{Input}: Policy model $\pi_\theta$, reference model $\pi_{\theta_{\text{ref}}}$, training dataset $\mathcal{D}$; \\
\textbf{Hyperparameters}: Learning rate $\eta$, KL coefficient $\beta$, group size $G$, temperature $\tau$, sharpening factor $\alpha$; \\
\textbf{Output}: Optimized policy model $\pi_\theta^*$;
\end{flushleft}
\begin{algorithmic}[1]
\For{step $t = 1$ to $T$}
    \State Sample query $q \sim \mathcal{D}$; 
    \State {/* --- \textit{\textbf{Step 1: Generation \& Belief Formation}} --- */}
    \State Sample $G$ responses $\mathcal{Y} = \{y_i\}_{i=1}^G \sim \pi_{\theta}(\cdot|q)$ and derive answers $\{a_i\}_{i=1}^G$;    
    \State Identify $M$ unique answers and compute frequencies $f_j$ for $j \in \{1, \dots, M\}$;
    \State Compute belief distribution $w_j \leftarrow \text{Softmax}(\alpha \cdot \log f_j)$; \Comment{Eq. (\ref{eq:posterior_weight})}
    
    \State {/* --- \textit{\textbf{Step 2: Dual-Objective Reward Calculation}} --- */}
    \State Compute Group Confidence: 
    $C_G \leftarrow \begin{cases} 
    1.0 & \text{if } M = 1 \\
    1 - \frac{H(W)}{\log M} = 1 - \frac{-\sum_{j} w_j \log w_j}{\log M} & \text{if } M > 1 
    \end{cases}$; \Comment{Eq. (\ref{eq:group_confidence})}
    
    \For{$i = 1$ to $G$}
        \State $r_{\text{cons}}(y_i) \leftarrow \begin{cases}
        1.0 & \text{if } a_i = \text{Vote}(\{a_k\}_{k=1}^G) \\
        0.0 & \text{otherwise}
    \end{cases}$; 
        \State $r_{\text{explore}}(y_i) \leftarrow \tanh (-\log w_i)$; 
        \State $R_i \leftarrow C_G \cdot r_{\text{cons}}(y_i) + (1 - C_G) \cdot r_{\text{explore}}(y_i)$; \Comment{Eq. (\ref{eq:bcr_final})}
    \EndFor
    
    \State {/* --- \textit{\textbf{Step 3: Adaptive Advantage Shaping (AAS)}} --- */}
    \State Compute statistics: $\mu_R \leftarrow \text{mean}(\{R_i\})$, $\sigma_R \leftarrow \text{std}(\{R_i\})$;
    \State Calculate skewness: $S \leftarrow \frac{1}{G} \sum_{i=1}^G \left(\frac{R_i - \mu_R}{\sigma_R + \epsilon}\right)^3$; \Comment{Eq. (\ref{eq:skewness})}
    
    \State Compute modulation weights: $w_{\text{pos}} \leftarrow \sigma(-S)$, $w_{\text{neg}} \leftarrow \sigma(S)$; 
    
    \For{$i = 1$ to $G$}
        \State $\tilde{A}_i \leftarrow \frac{R_i - \mu_R}{\sigma_R + \epsilon}$;
        \If{$\tilde{A}_i > 0$}
            \State $\hat{A}_i \leftarrow w_{\text{pos}} \cdot \tilde{A}_i$; \Comment{Positive Advantage Modulation}
        \Else
            \State $\hat{A}_i \leftarrow w_{\text{neg}} \cdot \tilde{A}_i$; \Comment{Negative Advantage Modulation}
        \EndIf
    \EndFor
    
    \State {/* --- \textit{\textbf{Step 4: Policy Optimization}} --- */}
    \State $\mathcal{L}(\theta) \leftarrow \mathbb{E} \Big[\frac{1}{G} \sum_{i=1}^G \frac{1}{|o_i|} \sum_{t=1}^{|o_i|} 
    \min\Big( r_{i,t}(\theta) \hat{A}_{i},\,  
    \text{clip}(r_{i,t}(\theta), 1-\epsilon, 1+\epsilon) \hat{A}_{i} \Big) 
    - \beta D_{\text{KL}}(\pi_\theta \| \pi_{\text{ref}})\Big]$;
    \State Update parameters: $\theta \leftarrow \theta + \eta \nabla_\theta \mathcal{L}(\theta)$;
\EndFor
\end{algorithmic}
\end{algorithm*}

\section{Appendix C: Further Analysis}
\label{sec:appendix_b}

\newtheorem{theorem}{\bf Theorem}
\newtheorem{definition}{\bf Definition}
\newtheorem{proposition}{\bf Proposition}
\newtheorem{lemma}{\bf Lemma}
\newtheorem{argument}{\bf Argument}
\newtheorem{conclusion}{\bf Conclusion}
\newtheorem{Proof}{\bf Proof}
\newtheorem{assumption}{\bf Assumption}
\newtheorem{remark}{\bf Remark}
\newtheorem{corollary}{\bf Corollary}

\subsection{Theoretical Derivation: FER as Precision-Weighted Active Inference}
\label{sec:fep_derivation}

While FER is conceptually grounded in the \textit{Free Energy Principle} (FEP), it can also be expressed as a tractable approximation of the maximization of the \textit{Precision-Weighted Expected Free Energy} (EFE). The correspondence between the EFE objective and the proposed reward formulation is formalized through two operational assumptions.

\paragraph{General Objective.}
In Active Inference, agents minimize the expected free energy $G(\pi)$. Maximizing $-G(\pi)$ can be decomposed into two components: identifying preferred outcomes (\textit{Pragmatic Value}) and reducing uncertainty (\textit{Epistemic Value}) \cite{friston2020generative}:
\begin{equation}
\begin{aligned}
    -G(\pi) &\approx \underbrace{\mathbb{E}_{Q(o)}[\ln P(o \mid \theta_{goal})]}_{\text{Pragmatic}} \\ 
    &\quad + \underbrace{\mathbb{E}_{Q(o)}[D_{KL}(Q(s \mid o) \parallel Q(s))]}_{\text{Epistemic}}
\end{aligned}
\end{equation}
Here, $o$, $\theta_{goal}$, and $s$ denote the observed generated answers, the parameters of the goal distribution (\textit{i.e.}, prior preferences), and the latent reasoning states, respectively.

To adapt this formulation to unsupervised RL, we define the \textit{Precision-Weighted EFE Objective} using the precision $\beta \in [0,1]$:
\begin{equation}
    \label{eq:weighted_efe}
    J(\pi, \beta) = \beta \cdot (\text{Pragmatic}) + (1-\beta) \cdot (\text{Epistemic})
\end{equation}

\begin{definition}[Empirical Precision]
\label{def:precision}
The precision $\beta$ of the model’s current belief state is defined as the complement of the normalized entropy of the generated sample distribution $W$:
\begin{equation}
    \beta \coloneqq C_G = 
    \begin{cases} 
    1.0, & \text{if } M = 1, \\
    1 - \frac{H(W)}{\log M}, & \text{if } M > 1,
    \end{cases}
    \label{eq:group_confidence}
\end{equation}
where $M$ is the number of unique samples. This formulation provides a tractable estimator of the model’s consensus confidence.
\end{definition}

To make Eq.~\eqref{eq:weighted_efe} computationally tractable without ground-truth supervision, we introduce the following two approximations:

\begin{assumption}[Consensus-Truth Proxy]
\label{ass:consensus}
In the absence of ground-truth supervision, the mode of the current empirical distribution $y^*$ is assumed to approximate the optimal goal state. The pragmatic utility is modeled as a Dirac delta function centered on the consensus:
\begin{equation}
\begin{aligned}
     &P(o \mid \theta_{goal}) \approx \delta(o = y^*) \\
     &\implies \text{Utility}_{\text{prag}} = \mathbb{I}(o = y^*)
\end{aligned}
\end{equation}
\end{assumption}

\begin{assumption}[Satiable Information Gain]
\label{ass:satiable}
The epistemic value is associated with the surprisal of the outcome $-\ln P(o)$, where $P(o)$ is approximated by $w_i$ (Eq.~(\ref{eq:posterior_weight})). Since raw surprisal is unbounded, the utility of information is assumed to follow a saturation curve confined to $[0,1]$:
\begin{equation}
    \text{Utility}_{\text{epis}} \approx \tanh(-\ln w_i)
\end{equation}
\end{assumption}

Based on these assumptions, we now present the theorem that formally justifies the FER reward formulation.

\begin{theorem}[FER Derivation]
\label{thm:fer_derivation}
Under Assumptions \ref{ass:consensus} and \ref{ass:satiable}, maximizing the Precision-Weighted EFE objective $J(\pi, \beta)$ with empirical precision $\beta = C_G$ is equivalent to maximizing the expected FER reward $R_{\text{FER}}$.
\end{theorem}

\begin{proof}
Let the empirical probability of an outcome $o_i$ be $w_i$. Substituting into Eq.~\eqref{eq:weighted_efe}:

1. \textbf{Pragmatic Term:} By Assumption~\ref{ass:consensus}, the pragmatic component becomes the indicator function $\mathbb{I}(o_i = y^*)$, which corresponds to $r_{\text{cons}}$.

2. \textbf{Epistemic Term:} By Assumption~\ref{ass:satiable}, the epistemic component becomes $\tanh(-\ln w_i)$, corresponding to $r_{\text{explore}}$.

3. \textbf{Precision Modulation:} By Definition~\ref{def:precision}, the weighting factor $\beta$ is given by the group confidence $C_G$.

Substituting these into $J$:
\begin{equation}
    \begin{aligned}
        J(\pi, C_G) &= C_G \cdot \mathbb{E}[\mathbb{I}(o = y^*)] \\
        &\quad + (1-C_G) \cdot \mathbb{E}[\tanh(-\ln w_i)] \\
        &= \mathbb{E}_{\pi} \left[ \underbrace{C_G \cdot r_{\text{cons}} + (1-C_G) \cdot r_{\text{explore}}}_{R_{\text{FER}}} \right]
    \end{aligned}
\end{equation}

Therefore, the reward $R_{\text{FER}}$ used in policy updates is a direct realization of the Precision-Weighted EFE objective.
\end{proof}

This derivation establishes FER as a concrete instantiation of Active Inference, in which the agent adaptively shifts from \textit{risk minimization} (consensus) to \textit{information seeking} (exploration) according to the statistical reliability ($C_G$) of its outputs. This theoretical connection provides a principled foundation for FER’s adaptive balance between consensus and exploration in unsupervised reasoning.

\begin{table*}[h]
\centering
\small
\renewcommand{\arraystretch}{1.3}
\begin{tabular}{p{0.95\textwidth}}
\toprule
\rowcolor[rgb]{0.88, 0.95, 1.0} \textbf{Setup:} Sample Size $G=8$. \textbf{Scenario:} A high-ambiguity task with a weak false consensus. \\
\rowcolor[rgb]{0.88, 0.95, 1.0} \textbf{Generated Answers:} $A_1 \sim A_5$ are incorrect (Weak Majority). $A_6 \sim A_8$ are correct (Strong Minority). \\
\rowcolor[rgb]{0.88, 0.95, 1.0} \textbf{Empirical Frequencies:} $f_{wrong}=5/8=0.625$, $f_{correct}=3/8=0.375$. \\
\midrule
\textbf{1. Belief Formation (Step 1, with $\alpha=2.0$):} \\
- Raw Proportions ($f$): $\{0.625, 0.375\}$. \\
- Squared Proportions ($f^\alpha$): $\{0.625^2 \approx 0.39, \ 0.375^2 \approx 0.14\}$. Sum $\Sigma \approx 0.53$. \\
- Belief Weights ($w = f^\alpha / \Sigma$): \\
  \quad $w_{wrong} = 0.39 / 0.53 \approx \textbf{0.74}$, \quad $w_{correct} = 0.14 / 0.53 \approx \textbf{0.26}$. \\

\textbf{2. Group Confidence ($C_G$) Calculation:} \\
- Entropy $H(W) = -(0.74 \ln 0.74 + 0.26 \ln 0.26) \approx 0.57$. \\
- Max Entropy $\ln M = \ln 2 \approx 0.69$ (2 unique answer clusters). \\
- $C_G = 1 - (0.57 / 0.69) = \textbf{0.17}$ (Low confidence, high uncertainty). \\

\textbf{3. Reward Component Calculation:} \\
- \textbf{consensus ($r_{\text{cons}}$):} Mode is ``Wrong''. $r^{wrong}_{\text{exploit}} = 1.0$, $r^{correct}_{\text{exploit}} = 0.0$. \\
- \textbf{Exploration ($r_{\text{explore}}$):} Using $\tanh (-\ln w)$. \\
  \quad $r^{wrong}_{\text{explore}} = \tanh (-\ln(0.74)) \approx \textbf{0.29}$ (Low information gain). \\
  \quad $r^{correct}_{\text{explore}} = \tanh (-\ln(0.26)) \approx \textbf{0.87}$ (High information gain). \\

\textbf{4. Final Reward Integration ($R = C_G \cdot r_{\text{cons}} + (1-C_G) \cdot r_{\text{explore}}$):} \\
- \textbf{Wrong Majority ($A_1 \sim A_5$):} $R_{wrong} = 0.17 \cdot 1.0 + 0.83 \cdot 0.29 \approx \textbf{0.41}$. \\
- \textbf{Correct Minority ($A_6 \sim A_8$):} \quad $R_{correct} = 0.17 \cdot 0.0 + 0.83 \cdot 0.87 \approx \textbf{0.72}$. \\

\textbf{5. Comparison with Standard Baseline (Majority Voting):} \\
- \textbf{Standard Voting Reward:} $R_{wrong} = 1.0$ (Win), $R_{correct} = 0.0$ (Lose). \\
\midrule
\rowcolor[rgb]{1.0, 0.95, 1.0} \textbf{Conclusion:} In standard voting, the model reinforces the incorrect majority ($R=1.0$). In FREIA, the high conflict ($C_G=0.17$) triggers a shift to exploration. Consequently, the correct minority receives a significantly higher reward than the majority (\textbf{0.72} vs \textbf{0.41}), successfully reversing the policy update direction towards the truth. \\
\bottomrule
\end{tabular}
\caption{A case study on False Consensus. FREIA detects the high uncertainty and inverts the reward signal, prioritizing the correct minority over the incorrect majority.}
\label{tab:freia_case_study}
\end{table*}

\subsection{The Necessity of Adaptive Precision via Group Confidence}
\label{app:mixing_ablation}

To validate the efficacy of FEP, we conducted an ablation study where the Group Confidence ($C_G$) was replaced with static mixing coefficients.

\subsubsection{Experimental Setup}
In the standard FREIA framework, the final reward is computed as an adaptive convex combination:
\begin{equation}
    R_{\text{total}} = C_G \cdot r_{\text{cons}} + (1 - C_G) \cdot r_{\text{explore}}
\end{equation}
where $C_G \in [0, 1]$ represents a confidence score indicating the degree of group consensus for each sample. In this ablation, $C_G$ was replaced with a fixed hyperparameter $\lambda \in \{0.2, 0.5, 0.8\}$ for all training samples:
\begin{equation}
    R_{\text{fixed}}(\lambda) = \lambda \cdot r_{\text{cons}} + (1 - \lambda) \cdot r_{\text{explore}}
\end{equation}

This setup forces a static trade-off, where $\lambda=0.2$ and $\lambda=0.8$ represent a high-exploration strategy and a high-consensus strategy, respectively.

\subsubsection{Analysis of Static vs. Dynamic Strategies}
Experimental results in Table~\ref{tab:mixing_ablation_revise} reveal a clear performance hierarchy and expose the weaknesses of static weighting:

\paragraph{Excessive Exploration ($\lambda = 0.2$).}
A low $\lambda=0.2$ yielded the poorest average performance. By penalizing convergence toward consensus, this configuration injected noise into learning signals, even when the consensus was correct. Maximizing surprisal without adequate consensus consistently led to divergent behavior.

\paragraph{Blind Conformity ($\lambda = 0.8$).}
Greater reliance on consensus improved overall performance but introduced systematic risk. When the majority was wrong, a high $\lambda$ reinforced erroneous reasoning. Without the ability to down‑weight consensus under uncertainty, the model remained vulnerable to majority‑driven bias. Therefore, a fixed $\lambda$ implicitly assumes a constant level of uncertainty across all problems, which is theoretically unsound.

\paragraph{Dynamic Adaptation ($C_G$).}
FREIA achieved the highest accuracy across almost all benchmarks by employing the adaptive $C_G$ as an \textit{intelligent gate}:
\begin{itemize}
    \item \textbf{Low Uncertainty ($C_G \to 1$):} Under strong consensus, FREIA operated close to the $\lambda=1.0$ case, fully exploiting the correct path.
    \item \textbf{High Uncertainty ($C_G \to 0$):} Under low consensus, FREIA automatically lowered the weight of $r_{\text{cons}}$ and amplified $r_{\text{explore}}$, enabling exploration of alternative reasoning paths rather than overfitting to unreliable majority answers.
\end{itemize}

\subsection{Theoretical Analysis of AAS} \label{appendix:theoretical_guarantees}

In this section, an analytical framework is provided to interpret the stability mechanisms in AAS. The gradient dynamics are analyzed under stylized regimes of reward skewness. This analysis illustrates that AAS functions as a variance‑reduction mechanism in practice.

The behavior of AAS is examined from two complementary perspectives: \textit{Influence Limitation} for outlier robustness in positive‑skew regimes, and \textit{Variance Damping} for stability in negative‑skew regimes.

\begin{proposition}[Gradient Stabilization in Skewed Regimes]
Under extreme skewness, AAS analytically reduces the gradient magnitudes of outliers in positive‑skew regimes and mitigates variance amplification in negative‑skew regimes.
\end{proposition}

\begin{proof}
We analyze two typical scenarios of training dynamics:

\paragraph{Scenario 1: Positive Skew.}
Consider a batch with sparse high rewards (\textit{e.g.}, a single high-reward answer among many low-reward answers). The distribution exhibits large positive skew $S > 0$.

(1) \textbf{Standard Normalization Instability:} 
    Consider a batch of size $G$ where only one sample $y_{\text{rare}}$ receives a large reward $R_{\text{rare}}$, while others have near-zero rewards. 
    Using the standard normalization $\tilde{R}_i = \frac{R_i - \mu_R}{\sigma_R}$, the batch mean is 
    \(\mu_R \approx R_{\text{rare}}/G\), and the variance is dominated by this single outlier:
    \begin{equation}
       \sigma_R^2 \approx \frac{(R_{\text{rare}} - \mu_R)^2}{G} \approx \frac{R_{\text{rare}}^2}{G}
    \Rightarrow \sigma_R \approx \frac{R_{\text{rare}}}{\sqrt{G}} 
    \end{equation}
    
    Therefore, the normalized reward of the rare sample becomes:
    \begin{equation}
        \tilde{R}_{\text{rare}} \approx \frac{R_{\text{rare}}}{R_{\text{rare}}/\sqrt{G}} = \sqrt{G}
    \end{equation}

    This yields $\tilde{A}_{\text{rare}} \approx \sqrt{G}$, making the gradient 
    \(\nabla J \propto \sqrt{G} \cdot \nabla \ln \pi(y_{\text{rare}})\) overly dependent on one sample. 
    If the outlier is wrong, it triggers a destructive high-variance policy update.
    
(2) \textbf{AAS Damping:} Since $S>0$, AAS applies a scaling weight $w_{\text{pos}} = \sigma(-S)<1$, giving an effective advantage $\hat{A}_{\text{rare}} = w_{\text{pos}} \cdot \tilde{A}_{\text{rare}} < \tilde{A}_{\text{rare}}$. This bounds the outlier’s influence and keeps the policy update conservative in unsupervised settings.

\paragraph{Scenario 2: Negative Skew.}
Consider a regime where low-value rewards are rare. The batch contains predominantly high-reward answers ($r \approx 1$) with probability $1-p$, and occasional low-reward answers ($r \approx 0$) with probability $p \ll 1$. The skewness $S<0$.

(1) \textbf{Variance Sensitivity in Standard RL:}
    The standard deviation of rewards is $\sigma \approx \sqrt{p}$. The normalized advantage for the rare failure is:
    \begin{equation}
        \tilde{A}_{fail} = \frac{0 - (1-p)}{\sqrt{p}} \approx -\frac{1}{\sqrt{p}}
    \end{equation}

    As $p \to 0$ (\textit{i.e.}, the model approaches perfection), the magnitude $|\tilde{A}_{fail}| \to \infty$. This results in gradient magnitudes inversely proportional to $\sqrt{p}$, which can introduce high variance and destabilize policy optimization.

(2) \textbf{AAS Variance Damping:}
In the negatively skewed regime, the distribution of rewards is dominated by high-reward samples. 
AAS attenuates these negative outliers through a sigmoid-based weighting $w_{neg} = \sigma(S)$.
As $S$ decreases, $w_{neg}$ rapidly approaches zero, effectively damping the contribution of extreme negative advantages:
\begin{equation}
    \lim_{S \to -\infty} w_{neg} \cdot A_{fail} = 0
\end{equation}

As a result, this alleviates gradient explosion and promotes stable training.
\end{proof}

\subsection{The Dual Role of Belief Sharpening: Exploration Control and Noise Filtering}
\label{sec:belief}

In Section 4.1, the belief‑sharpening mechanism was defined as \(w_j \propto f_j^\alpha\), where \(\alpha\) acts as a structural parameter controlling the learning dynamics. This section investigates the influence of \(\alpha\) from two perspectives: (i) its isolated effect on \(r_{\text{explore}}\), and (ii) its theoretical grounding within Generalized Bayesian Inference and the bias–variance trade‑off.

\paragraph{Isolation of Exploration Dynamics.}

Let \(f_{\text{win}}\) denote the frequency of the dominant consensus answer. The exponent \(\alpha\) defines the reward landscape and modulates the incentive for policy deviation:
    
(1) \textbf{\(\alpha=1\) (Raw Surprisal):} This case corresponds to using the raw frequency distribution without sharpening (\(w_{\text{win}}=f_{\text{win}}\)). During early training or under challenging conditions, consensus is weak, producing a flat distribution and a high exploration reward for the dominant answer. This results in redundant exploration, where the model repeatedly examines known solutions rather than searching for novel reasoning trajectories.
        
(2) \textbf{High \(\alpha\) (Mode Collapse):} As \(\alpha \to \infty\), \(w_{\text{win}} \to 1\) and \(r_{\text{explore}}(w_{\text{win}}) \to 0\). While this reduces redundant exploration, it prematurely suppresses diversity. If the initial majority is incorrect, gradients associated with the correct minority vanish, causing convergence to a local optimum.

Empirical results show that \(\alpha=2\) achieves superior performance, indicating that the sharpened reward \(\tanh(-\log w_j)\) offers a more informative signal than the raw form \(\tanh(-\log f_j)\) for guiding exploration.

\paragraph{Principled Selection: Bayesian and Bias-Variance Perspectives.}
The role of $\alpha$ can be further grounded in the \textit{Power Posterior} framework of Generalized Bayesian Inference:
\begin{equation}
    P(s|D) \propto [P(D|s)]^{\eta} \cdot P(s)
\end{equation}

In FREIA, the observed frequency $f_j$ serves as the empirical likelihood $P(D|s)$, and the computed weight $w_j$ represents the posterior belief. The parameter $\alpha$ corresponds to the exponent $\eta$, functioning as a \textit{confidence temperature}.

\begin{itemize}
    \item \textbf{Noise Filtering (Variance Reduction):} In unsupervised reinforcement learning, raw frequencies (\(\alpha=1\)) are highly variable and noise‑sensitive. Increasing \(\alpha\) acts as a denoising mechanism, suppressing the long tail of stochastic errors and concentrating probability mass around the mode. This adjustment reduces the variance of the learning signal and stabilizes gradient estimation.
    
    \item \textbf{Gradient Preservation (Bias Control):} In contrast to extreme hard‑filtering (\(\alpha \to \infty\)), the setting \(\alpha=2\) retains a small but non‑zero bias. This configuration maintains adaptability to adjust incorrect consensus, ensuring that plausible minority solutions remain learnable despite imperfections in current consensus.
\end{itemize}

As a result, $\alpha$ functions as a controllable temperature parameter balancing exploration and noise suppression.

\subsection{Empirical Robustness of Skewness Estimation} \label{sec:skewness}

To validate the stability of AAS, we tracked the evolution of the skewness throughout training for each batch, as shown in Figure \ref{fig:skewness}. The trajectories reveal two critical insights regarding stability:
\begin{itemize}
    \item \textbf{Temporal Stability:} All models display a coherent downward trend with minimal high‑frequency fluctuations. This smooth evolution indicates that the estimation variance is negligible, implying that the skewness metric remains stable across training iterations.
    \item \textbf{Direction Persistence:} Throughout the entire training process, the skewness consistently stays within the negative domain. This persistent negativity provides a substantial margin from zero, confirming that the direction of AAS modulation is reliably determined by the intrinsic reward landscape. It also demonstrates statistical immunity to stochastic variations arising from individual samples.
\end{itemize}

\begin{figure}[h!]
    \centering
    \includegraphics[width=3in]{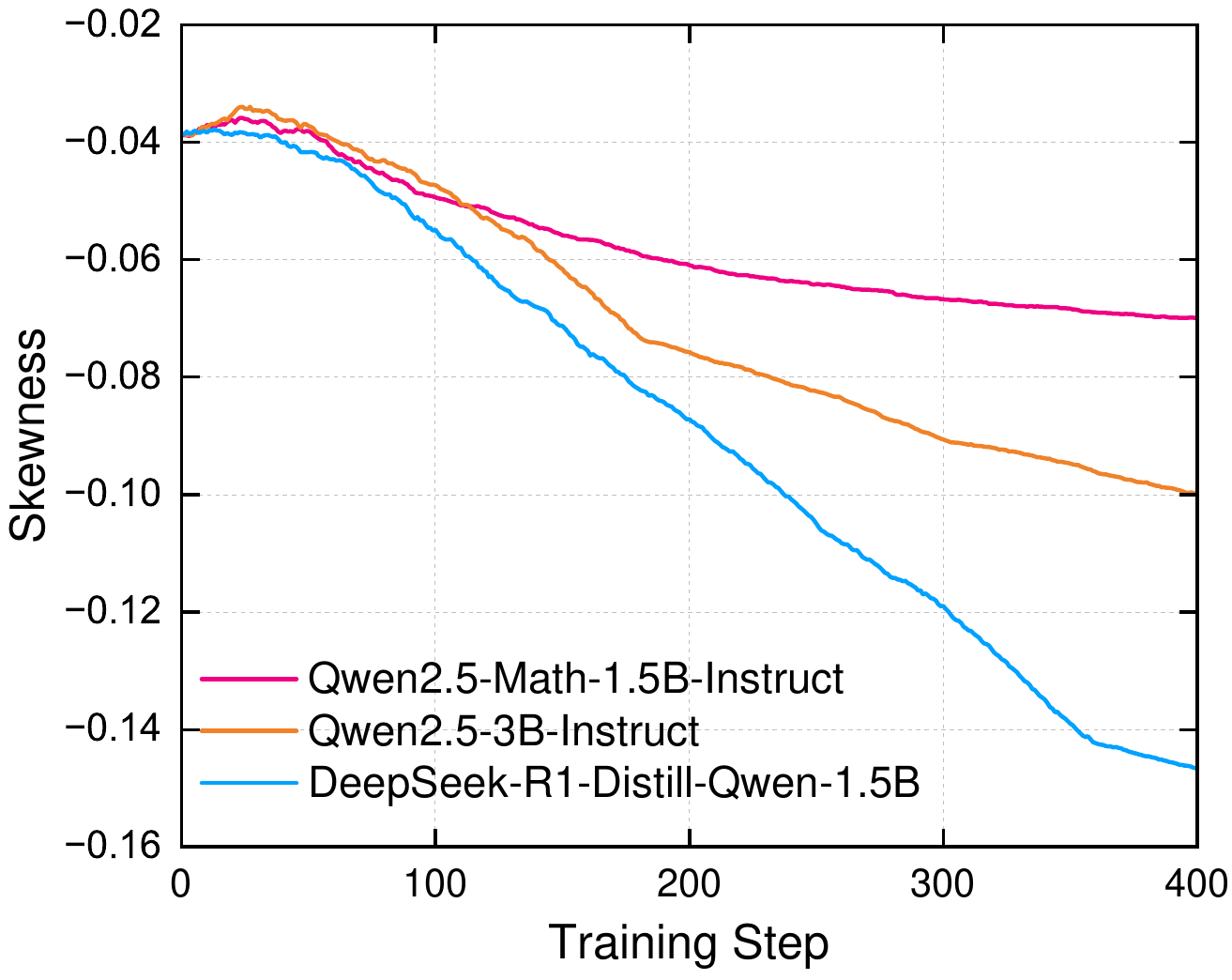}
    \caption{The evolution of the reward skewness throughout training.}
    \label{fig:skewness}
\end{figure}

\subsection{Empirical Analysis of Output Diversity}
\label{sec:diversity_analysis}

We evaluated the generative diversity of FREIA by comparing it with TTRL and GRPO on MATH500. Three complementary metrics were employed to characterize diversity across semantic, lexical, and logical dimensions.

(1) \textbf{Average Cosine Similarity (ACS):} ACS measures the semantic homogeneity among generated solutions. We computed cosine similarity between sentence embeddings\footnote{\url{https://huggingface.co/Qwen/Qwen3-Embedding-0.6B}} for all response pairs given the same question. A lower ACS value indicates greater semantic variance.

(2) \textbf{Self-BLEU:} This metric quantifies lexical overlap by calculating the average BLEU‑4 score of each generated solution against the remaining samples in a batch. Lower Self‑BLEU scores correspond to reduced verbatim repetition.

(3) \textbf{GPT-Judged Diversity Score:} We further assessed conceptual diversity using GPT‑4o as an external judge, which rated the distinctness of reasoning paths on a 0 $\sim$ 3 scale. Higher scores (3) were assigned to fundamentally different mathematical approaches, while lower scores (0) indicated superficial rephrasing.

\begin{table}[h]
    \centering
    \small
    \setlength{\tabcolsep}{1.6mm}
    \renewcommand{\arraystretch}{1.2}
    \begin{tabular}{l|ccc}
    \toprule
    \textbf{Method} & \textbf{ACS} ($\downarrow$) & \textbf{Self-BLEU} ($\downarrow$) & \textbf{GPT-Diversity} ($\uparrow$) \\
    \midrule
    GRPO            & 0.88            & 0.79                              & 1.05 \\
    TTRL            & 0.76                        & 0.68                              & 1.23 \\
    \textbf{FREIA} & \textbf{0.62}               & \textbf{0.51}                     & \textbf{1.56} \\
    \bottomrule
    \end{tabular}
    \caption{Diversity metrics on the MATH500 dataset using Qwen2.5-Math-1.5B-Instruct. $\downarrow$ indicates lower is better; $\uparrow$ indicates higher is better.}
    \label{tab:diversity_metrics}
\end{table}

As shown in Table \ref{tab:diversity_metrics}, GRPO attained the highest ACS and Self-BLEU scores, confirming its susceptibility to severe mode collapse. While TTRL alleviated this issue via consistency voting, it remained constrained by the dominance of the consensus term, limiting its ability to maintain high diversity. In contrast, FREIA achieved the lowest similarity scores across all metrics. These results demonstrated that the exploration reward term $r_{\text{explore}}$ effectively penalized redundancy in low‑confidence states, guiding the policy toward distinct reasoning trajectories rather than repetitive consensus. The observed improvements align with FEP introduced in Appendix~\ref{sec:fep_derivation}.

\paragraph{GPT-Judge Prompt.} The prompt used for diversity scoring is given as follows:

\textit{``Given the following set of correct solutions to a math problem, evaluate their diversity. Assign a score from 0 to 3: 
    0 = Almost identical wording; 
    1 = Different wording but same logic; 
    2 = Slightly different logical steps; 
    3 = Fundamentally different mathematical approaches. 
    Output only the single numerical score.''}

\subsection{Computational Efficiency Analysis}
\label{sec:efficiency_analysis}

This section analyzes the training efficiency of FREIA through the lens of time complexity.

In RL-based frameworks, the computational bottleneck is dominated by the LLM's forward pass (reasoning generation) and backward pass (gradient optimization). Let $N$, $L$, and $G$ denote the number of model parameters, sequence length, and group size (rollouts per prompt), respectively. The complexity of these backbone operations scales as $\mathcal{O}(G \cdot L \cdot N)$.

In contrast, the mechanisms introduced by FREIA operate solely on scalar reward values, bypassing high-dimensional tensor manipulations:
\begin{itemize}
    \item \textbf{\texorpdfstring{$C_G$}{CG} Calculation:} This involves sorting the reward distribution of size $G$, with a complexity of $\mathcal{O}(G \log G)$.
    \item \textbf{Skewness and AAS Weighting:} This requires computing statistical moments of the reward distribution, which is linear in group size (\textit{i.e.}, $\mathcal{O}(G)$).
\end{itemize}

Given that $G$ is typically small (\textit{e.g.}, $G=8$) while $N$ is in the billions, the inequality $\mathcal{O}(G \cdot L \cdot N) \gg \mathcal{O}(G \log G)$ holds in practice across different model scales. Therefore, the additional computational overhead introduced by FREIA is theoretically negligible.

\subsection{Sensitivity Analysis: The Impact of Group Size \texorpdfstring{$G$}{G}}
\label{sec:sensitivity_group_size}

The group size $G$ is a critical hyperparameter in FREIA, governing the fidelity of statistical estimates (\textit{i.e.}, skewness, $C_G$) and the computational cost. While larger groups provide more robust gradient signals, they also increase memory and time consumption. To determine the optimal trade-off, we conducted an ablation study using Qwen2.5-Math-1.5B-Instruct with group sizes $G \in \{4, 6, 8\}$. 

As shown in Table \ref{tab:group_size_ablation}, performance is positively correlated with $G$, confirming that larger populations stabilize skewness-based advantage estimation. However, a distinct trend of \textit{diminishing performance gains} is also observed. Since rollout generation costs scale linearly with $G$, scaling beyond $G=8$ yields insufficient gains to justify the added computational burden. Therefore, $G=8$ offered a balanced configuration that maintains statistical robustness while minimizing unnecessary computation.

\subsection{Case study} \label{sec:case}

\clearpage
\onecolumn

\ifdefined\nolinenumbers
  \nolinenumbers
\fi

\begin{tcolorbox}[
    colback=white, 
    colframe=black, 
    title=\textbf{Case Study: Reasoning Process Comparison Using Different Unsupervised Methods}, 
    coltitle=white, 
    colbacktitle=black!80, 
    breakable,
    skin=enhanced,
    boxrule=1pt, 
    width=\textwidth
]

\textbf{Original Question:} \\
The coordinates of a parallelogram are $(5,3)$, $(6,8)$, $(7,4)$ and $(x,y)$ and $x>7$. What is the value of $x+y$?

\tcbline

\textbf{FREIA Answer (Correct):} \textcolor{teal}{\checkmark} \\
To find the coordinates of the fourth vertex $(x,y)$, we must consider that the vertices can be connected in different orders to form a parallelogram. Let the given points be $A(5,3)$, $B(6,8)$, and $C(7,4)$. There are three possible positions for the fourth vertex $D(x,y)$, depending on which pair of points forms a diagonal.

\vspace{0.2cm}
\textbf{Case 1: $AB$ is a diagonal.} \\
The midpoint of $AB$ must equal the midpoint of $CD$.
\begin{align*}
\frac{5+6}{2} = \frac{7+x}{2} &\implies 11 = 7+x \implies x=4 \\
\frac{3+8}{2} = \frac{4+y}{2} &\implies 11 = 4+y \implies y=7
\end{align*}
Here $x=4$, which is not greater than 7. \textbf{This case is invalid.}

\vspace{0.2cm}
\textbf{Case 2: $AC$ is a diagonal.} \\
The midpoint of $AC$ must equal the midpoint of $BD$.
\begin{align*}
\frac{5+7}{2} = \frac{6+x}{2} &\implies 12 = 6+x \implies x=6 \\
\frac{3+4}{2} = \frac{8+y}{2} &\implies 7 = 8+y \implies y=-1
\end{align*}
Here $x=6$, which is not greater than 7. \textbf{This case is invalid.}

\vspace{0.2cm}
\textbf{Case 3: $BC$ is a diagonal.} \\
The midpoint of $BC$ must equal the midpoint of $AD$.
\begin{align*}
\frac{6+7}{2} = \frac{5+x}{2} &\implies 13 = 5+x \implies x=8 \\
\frac{8+4}{2} = \frac{3+y}{2} &\implies 12 = 3+y \implies y=9
\end{align*}
Here $x=8$, which satisfies the condition $x>7$.

\vspace{0.2cm}
Thus, the valid coordinates are $(8,9)$. \\
The value of $x+y = 8+9 = 17$. \\
\textbf{The final answer is \boxed{17}.}

\tcbline

\textbf{TTRL Answer (Incorrect -- Logic Error):} \textcolor{red}{\XSolidBrush} \\
Let the vertices of the parallelogram be $A(5,3)$, $B(6,8)$, $C(7,4)$ and $D(x,y)$ in counterclockwise order.
In a parallelogram, the diagonals bisect each other. This means the midpoint of $AC$ is the same as the midpoint of $BD$.

Midpoint of $AC$:
\[ x_m = \frac{5+7}{2} = 6, \quad y_m = \frac{3+4}{2} = 3.5 \]

Now set this equal to the midpoint of $BD$:
\begin{align*}
6 = \frac{6+x}{2} &\implies 12 = 6+x \implies x=6 \\
3.5 = \frac{8+y}{2} &\implies 7 = 8+y \implies y=-1
\end{align*}

So the point is $(6,-1)$.
The question asks for $x+y$.
\[ x+y = 6+(-1) = 5 \]
\textbf{The final answer is \boxed{5}.} \\

\textit{(Note: TTRL fails because it defaults to the standard vertex ordering assumption and ignores the explicit $x>7$ constraint, a common failure mode in consensus-based methods where the majority of rollouts miss the subtle constraint.)}

\tcbline

\textbf{Entropy Answer (Collapse -- Repetition Loop):} \textcolor{red}{\XSolidBrush} \\
To solve this problem, we use the property that the diagonals of a parallelogram bisect each other. \\
Let the points be $A(5,3)$, $B(6,8)$, $C(7,4)$.
We need to find $D(x,y)$. \\
The midpoint of the diagonal is the average of the coordinates. \\
The midpoint of the diagonal is the average of the coordinates. \\
The midpoint of the diagonal is the average of the coordinates. \\
The midpoint of the diagonal is the average of the coordinates. \\
The midpoint of the diagonal is the average of the coordinates. \\
The midpoint of the diagonal is the average of the coordinates. \\
The midpoint of the diagonal is the average of the coordinates... \\

\textit{(Note: The model enters a degenerate loop, repeating the same phrase endlessly, a typical collapse symptom in entropy minimization.)}

\tcbline

\textbf{Intuitor Answer (Collapse -- Degenerate Reasoning):} \textcolor{red}{\XSolidBrush} \\
Step 1: Identify the coordinates.
The coordinates are $(5,3)$, $(6,8)$, $(7,4)$ and $(x,y)$. \\
Step 2: Use the formula.
We use the formula for the parallelogram. \\
Step 1: Identify the coordinates.
The coordinates are $(5,3)$, $(6,8)$, $(7,4)$ and $(x,y)$. \\  
Step 2: Use the formula.
We use the formula for the parallelogram. \\
Step 1: Identify the coordinates.
The coordinates are $(5,3)$, $(6,8)$, $(7,4)$ and $(x,y)$. \\
Step 2: Use the formula... \\

\textit{(Note: Similar to Entropy, Intuitor (Self-Certainty) collapses into a high-confidence repetition loop without performing actual calculation.)}

\end{tcolorbox}

\section{Appendix D: Additional Experimental Results} \label{sec:additional results}

\begin{table*}[t]
\centering
\small
\setlength{\tabcolsep}{7mm}
\begin{tabular}{lcccc}
\toprule
\textbf{Method} & \textbf{Geometry3K} & \textbf{Spider} & \textbf{BIRD} & \textbf{Avg.} \\
\midrule
Base     & 25.7              & 70.2              & 43.6              & 46.5 \\
GRPO     & 35.6$_{\pm 1.1}$  & 73.8$_{\pm 0.5}$  & 55.7$_{\pm 0.9}$  & 55.0$_{\pm 0.8}$ \\
TTRL     & 35.3$_{\pm 0.8}$  & 74.0$_{\pm 0.4}$  & 56.0$_{\pm 0.2}$  & 55.1$_{\pm 0.4}$ \\
Intuitor & 34.8$_{\pm 0.9}$  & 72.8$_{\pm 0.6}$  & 54.5$_{\pm 0.8}$  & 54.0$_{\pm 0.8}$ \\
Entropy  & 34.4$_{\pm 0.7}$  & 72.5$_{\pm 0.5}$  & 54.0$_{\pm 0.7}$  & 53.6$_{\pm 0.6}$ \\
FREIA    & \textbf{36.1}$_{\pm 0.5}$ & \textbf{74.4}$_{\pm 0.2}$ & \textbf{56.4}$_{\pm 0.3}$ & \textbf{55.6}$_{\pm 0.3}$ \\
\bottomrule
\end{tabular}
\caption{Generalization performance on SQL generation (Spider, BIRD) and multi-modal reasoning (Geometry3K). The results are reported as mean and standard deviation across 3 random seeds (format: $\text{Mean}_{\pm \text{Std}}$).}
\label{tab:generalization_results}
\end{table*}

\begin{table*}[h]
\centering
\small
\setlength{\tabcolsep}{1.5mm}
\renewcommand{\arraystretch}{1.1}
\begin{tabular}{llccccccc}
\toprule
\textbf{Model} & \textbf{Method} & \textbf{MATH500} & \textbf{AIME24} & \textbf{AIME25} & \textbf{AMC23} & \textbf{Minerva} & \textbf{Olympiad} & \textbf{Avg.} \\
\midrule
\multirow{4}{*}{\textbf{Qwen2.5-Math-1.5B}} 
 & w/o Consensus & 74.4 & 10.0 & 6.7 & 47.5 & 29.4 & 38.9 & 34.5 \\
 & w/o Exploration   & 74.8 & 13.3 & 10.0 & 50.0 & 30.1 & 39.5 & 36.3 \\
 & w/o AAS     & 75.2 & 13.3 & 13.3 & 50.0 & 31.3 & 40.5 & 37.3 \\
 & \textbf{FREIA} & \textbf{75.4} & \textbf{13.3} & \textbf{16.7} & \textbf{52.5} & \textbf{32.0} & \textbf{41.2} & \textbf{38.5} \\
\midrule
\multirow{4}{*}{\textbf{Qwen2.5-3B-Instruct}} 
 & w/o Consensus & 63.2 & 3.3 & 3.3 & 35.0 & 24.6 & 29.5 & 26.5 \\
 & w/o Exploration   & 64.2 & 6.7 & 6.7 & 37.5 & 25.4 & 30.7 & 28.5 \\
 & w/o AAS     & 64.8 & 6.7 & 6.7 & \textbf{40.0} & 26.1 & 31.5 & 29.3 \\
 & \textbf{FREIA} & \textbf{65.2} & \textbf{10.0} & \textbf{10.0} & 37.5 & \textbf{26.1} & \textbf{31.9} & \textbf{30.1} \\
\midrule
\multirow{4}{*}{\textbf{DeepSeek-R1-Distill}} 
 & w/o Consensus & 81.2 & 16.7 & 16.7 & 65.0 & 29.0 & 47.2 & 42.6 \\
 & w/o Exploration   & 81.6 & 16.7 & 16.7 & 67.5 & 29.8 & 48.5 & 43.5 \\
 & w/o AAS     & 82.0 & 20.0 & 16.7 & 70.0 & 30.5 & \textbf{49.6} & 44.8 \\
 & \textbf{FREIA} & \textbf{82.2} & \textbf{20.0} & \textbf{20.0} & \textbf{72.5} & \textbf{31.3} & 49.4 & \textbf{45.9} \\
\bottomrule
\end{tabular}
\caption{Detailed Ablation Study (RQ3). We compare FREIA against removing key components: w/o Consensus, w/o Exploration, and w/o AAS.}
\label{tab:ablation_full_datasets}
\end{table*}

\begin{table*}[h]
\centering
\small
\setlength{\tabcolsep}{1.5mm}
\renewcommand{\arraystretch}{1.1}
\begin{tabular}{llccccccc}
\toprule
\textbf{Model} & \textbf{Coefficient} & \textbf{MATH500} & \textbf{AIME24} & \textbf{AIME25} & \textbf{AMC23} & \textbf{Minerva} & \textbf{Olympiad} & \textbf{Avg.} \\
\midrule
\multirow{5}{*}{\textbf{Qwen2.5-Math-1.5B}} 
 & $\alpha=0.5$ & 74.8 & 13.3 & 16.7 & 50.0 & 31.3 & 40.1 & 37.7 \\
 & $\alpha=1.0$ & 75.2 & 13.3 & 16.7 & 52.5 & 31.6 & 40.8 & 38.4 \\
 & \textbf{$\alpha=2.0$} & \textbf{75.4} & \textbf{13.3} & \textbf{16.7} & \textbf{52.5} & \textbf{32.0} & \textbf{41.2} & \textbf{38.5} \\
 & $\alpha=3.0$ & 75.0 & 13.3 & 16.7 & 50.0 & 32.0 & 40.9 & 38.0 \\
 & $\alpha=4.0$ & 74.8 & 13.3 & 13.3 & 50.0 & 31.3 & 40.2 & 37.2 \\
\midrule
\multirow{5}{*}{\textbf{Qwen2.5-3B-Instruct}} 
 & $\alpha=0.5$ & 64.6 & 6.7 & 6.7 & 37.5 & 25.4 & 30.6 & 28.6 \\
 & $\alpha=1.0$ & 65.0 & 10.0 & 10.0 & 37.5 & 25.7 & 31.8 & 30.0 \\
 & \textbf{$\alpha=2.0$} & \textbf{65.2} & \textbf{10.0} & \textbf{10.0} & 37.5 & \textbf{26.1} & \textbf{31.9} & \textbf{30.1} \\
 & $\alpha=3.0$ & 64.8 & 10.0 & 6.7 & 40.0 & 25.7 & 31.9 & 30.0 \\
 & $\alpha=4.0$ & 64.6 & 10.0 & 6.7 & 37.5 & 25.0 & 31.5 & 29.2 \\
\midrule
\multirow{5}{*}{\textbf{DeepSeek-R1-Distill}} 
 & $\alpha=0.5$ & 81.6 & 16.7 & 20.0 & 70.0 & 30.5 & 49.3 & 44.7 \\
 & $\alpha=1.0$ & 82.0 & 16.7 & 20.0 & 72.5 & 30.9 & 49.9 & 45.3 \\
 & \textbf{$\alpha=2.0$} & \textbf{82.2} & \textbf{20.0} & \textbf{20.0} & \textbf{72.5} & \textbf{31.3} & \textbf{49.4} & \textbf{45.9} \\
 & $\alpha=3.0$ & 81.8 & 20.0 & 16.7 & 72.5 & 31.3 & 50.1 & 45.4 \\
 & $\alpha=4.0$ & 81.6 & 16.7 & 16.7 & 70.0 & 30.9 & 49.6 & 44.3 \\
\bottomrule
\end{tabular}
\caption{Hyperparameter Sensitivity Analysis (RQ4). Impact of the coefficient $\alpha$ on model performance across diverse benchmarks.}
\label{tab:sensitivity_robustness_final}
\end{table*}

\begin{table*}[t]
\centering
\resizebox{\textwidth}{!}{
\begin{tabular}{lcccccccc}
\toprule
\textbf{Setting} & \textbf{Mixing Strategy} & \textbf{MATH500} & \textbf{AIME24} & \textbf{AIME25} & \textbf{AMC23} & \textbf{Minerva} & \textbf{Olympiad} & \textbf{Avg.} \\
\midrule
Fixed Mixing & $\lambda=0.2$ & 74.6 & 10.0 & 13.3 & 48.5 & 30.5 & 39.8 & 36.1 \\
Fixed Mixing & $\lambda=0.5$ & 74.8 & 13.3 & 13.3 & 50.0 & 30.8 & 40.2 & 37.1 \\
Fixed Mixing & $\lambda=0.8$ & 75.0 & 13.3 & 16.7 & 50.0 & 31.3 & 40.7 & 37.8 \\
\midrule
\textbf{FREIA} & \textbf{Dynamic $C_G$} & \textbf{75.4} & \textbf{13.3} & \textbf{16.7} & \textbf{52.5} & \textbf{32.0} & \textbf{41.2} & \textbf{38.5} \\
\bottomrule
\end{tabular}
}
\caption{Ablation study on the mixing coefficient $\lambda$ using Qwen2.5-Math-1.5B-Instruct. Baselines with fixed $\lambda$ represent a static trade-off between Consensus ($\lambda$) and Surprisal ($1-\lambda$). FREIA uses Group Confidence ($C_G$) to adaptively modulate this trade-off per instance.}
\label{tab:mixing_ablation_revise}
\end{table*}

\begin{table*}[h]
\centering
\small
\begin{tabular}{lccccccccc}
\toprule
\textbf{Setting} & \textbf{MATH500} & \textbf{AIME24} & \textbf{AIME25} & \textbf{AMC23} & \textbf{Minerva} & \textbf{Olympiad} & \textbf{Avg.} & \textbf{Growth Rate} \\
\midrule
FREIA ($G=4$) & 74.6 & 13.3 & 13.3 & 47.5 & 30.5 & 40.1 & 36.6 & - \\
FREIA ($G=6$) & 75.2 & 13.3 & 16.7 & 50.0 & 31.6 & 41.1 & 38.0 & \textcolor{blue}{+1.4\%} \\
FREIA ($G=8$) & \textbf{75.4} & \textbf{13.3} & \textbf{16.7} & \textbf{52.5} & \textbf{32.0} & \textbf{41.2} & \textbf{38.5} & \textcolor{teal}{+0.5\%} \\
\bottomrule
\end{tabular}
\caption{Ablation study on Group Size $G$ using Qwen2.5-Math-1.5B-Instruct. We report the Pass@1 score across all datasets. Growth Rate indicates the performance gain relative to the previous group size setting.}
\label{tab:group_size_ablation}
\end{table*}

\end{document}